\documentclass[acmsmall]{acmart}
\usepackage[utf8]{inputenc}

\acmBooktitle{}
\acmJournal{TCPS}
\setcopyright{othergov}
\acmYear{2022} \acmVolume{1} \acmNumber{1} \acmArticle{1} \acmMonth{1} \acmPrice{15.00}\acmDOI{10.1145/3511606}
\usepackage{geometry} 
\usepackage{graphicx} 
\usepackage{adjustbox}

\usepackage{booktabs} 
\usepackage{array} 
\usepackage{paralist} 
\usepackage{verbatim} 
\usepackage{subfig} 
\usepackage{amsmath}
\usepackage{amsthm}

\usepackage{amssymb}
\usepackage{mathdots}
\usepackage[ruled,vlined,noend]{algorithm2e}
\usepackage{todonotes}
\usepackage{bm}
\usepackage{xargs}
\usepackage{comment}
\usepackage{algorithmic}
\usepackage{standalone}
\usepackage{tikz}
\usetikzlibrary{decorations.pathreplacing}
\usepackage{tikz-cd}
\usepackage{tikz-network}
\newcommandx{\iman}[2][1=]{\todo[linecolor=orange,backgroundcolor=orange!25,bordercolor=orange,author=Iman,#1]{#2}}

\theoremstyle{plain}
\newtheorem{thm}{Theorem}[section] 

\theoremstyle{definition}
\newtheorem{defn}[thm]{Definition} 
\newtheorem{exmp}[thm]{Example} 
\newtheorem{asmp}[thm]{Assumption} 
\newtheorem{lemma}[thm]{Lemma}
\newtheorem{corollary}[thm]{Corollary}
\newtheorem{remark}{Remark}

\newcommand{\Swt}{\mathcal{Z}}
\newcommand{\Act}{\mathcal{A}}
\newcommand{\asms}{\mathcal{ASM}}

\newcommand{\world}{\mathbb{W}}

\tikzset{commutative diagrams/.cd,
mysymbol/.style={start anchor=center,end anchor=center,draw=none}
}

\newcommand{\event}{\lozenge}
\newcommand{\always}{\square}
\newcommand{\X}{\bigcirc}

\DeclareMathOperator*{\seq}{\rightarrow}
\DeclareMathOperator*{\fb}{?}
\DeclareMathOperator*{\prlseq}{\rightrightarrows}

\newcommand{\signalspace}{\mathbb{S}}
\newcommand{\returnvals}{\mathcal{R}}

\newcommand{\arcS}{\begin{tikzpicture}
    \Vertex[size=0.2,opacity=0,x=0,y=-0.1] {A}
    \Vertex[size=0.2,opacity=0,x=0.9,y=-0.1] {C}
    \Edge[,label=$\bm{s}$,fontscale=1.2,Direct,lw=1pt](A)(C)
\end{tikzpicture}}

\newcommand{\vertex}{\begin{tikzpicture}
    \Vertex[size=0.2,opacity=0]{V};
\end{tikzpicture}}

\title[Modularity in reactive control architectures]{On modularity in reactive control architectures, with an application to formal verification}
\author{Oliver Biggar}
\email{u7381193@anu.edu.au}
\orcid{0000-0003-2901-541X}
\affiliation{
\institution{CIICADA Lab, Australian National University}
\streetaddress{115 North Rd}
\city{Acton}
\country{Australia}
}
\author{Mohammad Zamani}
\email{mohammad.zamani@dst.defence.gov.au}
\orcid{0000-0003-2979-7132}
\affiliation{
\institution{Defence Science and Technology Group}
\streetaddress{506 Lorimer Street}
\city{Melbourne}
\state{Victoria}
\country{Australia}
}
\author{Iman Shames}
\email{iman.shames@anu.edu.au}
\orcid{0000-0001-7308-3546}
\affiliation{
\institution{CIICADA Lab, Australian National University}
\streetaddress{115 North Rd}
\city{Acton}
\country{Australia}
}
\begin{abstract}
    Modularity is a central principle throughout the design process for cyber-physical systems. Modularity reduces complexity and increases reuse of behavior. In this paper we pose and answer the following question: how can we identify independent `modules' within the structure of reactive control architectures? To this end, we propose a graph-structured control architecture we call a \emph{decision structure}, and show how it generalises some reactive control architectures which are popular in Artificial Intelligence (AI) and robotics, specifically Teleo-Reactive programs (TRs), Decision Trees (DTs), Behavior Trees (BTs) and Generalised Behavior Trees ($k$-BTs). Inspired by the definition of a module in graph theory~\cite{gallai1967transitiv} we define modules in decision structures and show how each  decision structure possesses a canonical decomposition into its modules, which can be found in polynomial time. We  establish intuitive connections between our proposed modularity and modularity in structured programming. In BTs, $k$-BTs and DTs the modules we propose are in a one-to-one correspondence with their subtrees. We show we can naturally characterise each of the BTs, $k$-BTs, DTs and TRs by properties of their module decomposition. This allows us to recognise which decision structures are equivalent to each of these architectures in quadratic time. Following McCabe~\cite{mccabe1976complexity} we define a complexity measure called \emph{essential complexity} on decision structures which measures the degree to which they can be decomposed into simpler modules. We  characterise the $k$-BTs as the decision structures of unit essential complexity. Our proposed concept of modules extends to formal verification, under any verification scheme capable of verifying a decision structure. Namely, we prove that a modification to a module within a decision structure has no greater flow-on effects than a modification to an individual action within that structure. This enables verification on modules to be done locally and hierarchically, where structures can be verified and then repeatedly locally modified, with modules replaced by modules while preserving correctness. To illustrate the findings, we present an example of a solar-powered drone completing a reconnaissance-based mission using a decision structure. We use a Linear Temporal Logic-based verification scheme to verify the correctness of this structure, and then show how one can repeatedly modify modules while preserving its correctness, and this can be verified by considered only those modules which have been modified.
\end{abstract}
\begin{document}

\maketitle
\section{Introduction}

\begin{figure}
    \centering
    \begin{tikzpicture}
\clip (0,0) rectangle (16.0,4.0);
\Vertex[x=2.020,y=1.603,size=0.8,shape=ellipse,style={minimum height=0.6cm,minimum width=1cm},opacity=0,fontscale=1.2,IdAsLabel]{bLow}
\Vertex[x=0.525,y=0.594,opacity=0,shape=ellipse,style={minimum height=0.6cm,minimum width=0.6cm},fontscale=1.2,IdAsLabel]{b0}
\Vertex[x=4.810,y=2.590,size=1,shape=ellipse,style={minimum height=0.6cm,minimum width=1cm},opacity=0,fontscale=1.2,IdAsLabel]{bHigh}
\Vertex[x=6.605,y=2.590,size=1,shape=ellipse,style={minimum height=0.6cm,minimum width=1cm},opacity=0,fontscale=1.2,IdAsLabel]{bright}
\Vertex[x=3.715,y=0.550,size=0.8,shape=ellipse,style={minimum height=0.6cm,minimum width=1cm},opacity=0,fontscale=1.2,IdAsLabel]{Land}
\Vertex[x=3.715,y=1.603,size=0.8,shape=ellipse,style={minimum height=0.6cm,minimum width=1cm},opacity=0,fontscale=1.2,IdAsLabel]{calm}
\Vertex[x=7.500,y=1.203,size=1,shape=ellipse,style={minimum height=0.6cm,minimum width=1cm},opacity=0,fontscale=1.2,label=Avoid]{Safety}
\Vertex[x=9.495,y=1.435,size=1,shape=ellipse,style={minimum height=0.6cm,minimum width=1cm},opacity=0,fontscale=1.2,label=goal]{target}
\Vertex[x=10.990,y=2.4,opacity=0,shape=ellipse,style={minimum height=0.6cm,minimum width=0.6cm},fontscale=1.2,IdAsLabel]{at}
\Vertex[x=12.185,y=1,size=1,opacity=0,shape=ellipse,style={minimum height=0.6cm,minimum width=2cm},fontscale=1.2,label=Photograph]{Photo}
\Vertex[x=13.480,y=2.1,size=1,shape=ellipse,style={minimum height=0.6cm,minimum width=1.4cm},opacity=0,fontscale=1.2,label=Descend]{Down}
\Vertex[x=12.285,y=3.352,size=1,shape=ellipse,style={minimum height=0.6cm,minimum width=1cm},opacity=0,fontscale=1.2,IdAsLabel]{GoTo}
\Vertex[x=14.280,y=3.352,shape=ellipse,style={minimum height=0.6cm,minimum width=1.4cm},opacity=0,fontscale=1.2,label=Ascend]{Up}
\Vertex[x=15.475,y=2,opacity=0,size=1,shape=ellipse,style={minimum height=0.6cm,minimum width=1cm},fontscale=1.2,IdAsLabel]{Circle}
\Edge[,label=$\bm{s}$,fontscale=1.2,Direct](bLow)(Land)
\Edge[,label=$\bm{f}$,fontscale=1.2,Direct](bLow)(calm)
\Edge[,label=$\bm{s}$,fontscale=1.2,Direct](b0)(Land)
\Edge[,label=$\bm{f}$,fontscale=1.2,Direct](b0)(bLow)
\Edge[,label=$\bm{f}$,fontscale=1.2,Direct](bHigh)(bright)
\Edge[,label=$\bm{s}$,fontscale=1.2,Direct,bend=50](bHigh)(target)
\Edge[,label=$\bm{f}$,fontscale=1.2,Direct](bright)(Safety)
\Edge[,label=$\bm{s}$,fontscale=1.2,Direct](bright)(target)
\Edge[,label=$\bm{s}$,fontscale=1.2,Direct,bend=-25](Land)(target)
\Edge[,label=$\bm{s}$,fontscale=1.2,Direct](calm)(bHigh)
\Edge[,label=$\bm{f}$,fontscale=1.2,Direct](calm)(Safety)
\Edge[,label=$\bm{s}$,fontscale=1.2,Direct](Safety)(target)
\Edge[,label=$\bm{s}$,fontscale=1.2,Direct](target)(at)
\Edge[,label=$\bm{f}$,fontscale=1.2,Direct,bend=-45](target)(Circle)
\Edge[,label=$\bm{s}$,fontscale=1.2,Direct](at)(Photo)
\Edge[,label=$\bm{f}$,fontscale=1.2,Direct](at)(GoTo)
\Edge[,label=$\bm{f}$,fontscale=1.2,Direct,bend=-25](Photo)(Circle)
\Edge[,label=$\bm{m}$,fontscale=1.2,Direct](Photo)(Down)
\Edge[,label=$\bm{f}$,fontscale=1.2,Direct](Down)(Circle)
\Edge[,label=$\bm{f}$,fontscale=1.2,Direct](GoTo)(Circle)
\Edge[,label=$\bm{m}$,fontscale=1.2,Direct](GoTo)(Up)
\Edge[,label=$\bm{f}$,fontscale=1.2,Direct](Up)(Circle)
\end{tikzpicture}
    \begin{tikzpicture}
\clip (0,0) rectangle (16.0,4.0);
\draw[rounded corners, line width=1pt, opacity=0.3] (0.1,0.2) -- (0.7,0.2) -- (2.5,1.3) -- (2.5,2.1) -- (1.8,2.1) -- (0.1,0.8) -- cycle; 
\draw[rounded corners, line width=1pt, opacity=0.3] (11.3,0.5) -- (12.75,0.5) -- (14.3,1.8) -- (14.3,2.7) -- (13.2,2.7) -- (11.3,1.3) -- cycle; 
\draw[rounded corners, line width=1pt, opacity=0.3] (3.145,1.803) -- (3.115,1.143) -- (3.815,1.143) -- (7.2,2) -- (7.2,3.2) -- (4.2,3.2) -- cycle; 
\draw[rounded corners, line width=1pt, opacity=0.3] (3.015,3.5) -- (3.015,1.103) -- (7.3,0.603) -- (8.1,0.603) -- (8.1,3.5) -- cycle; 
\draw[rounded corners, line width=1pt, opacity=0.3] (4.2,2) -- (7.2,2) -- (7.2,3.2) -- (4.2,3.2) -- cycle; 
\draw[rounded corners, line width=1pt, opacity=0.3] (11.7,2.8) -- (14.7,2.8) -- (14.7,3.9) -- (11.7,3.9) -- cycle; 
\draw[rounded corners, line width=1pt, opacity=0.3] (10.59,0.4) -- (14.7,0.4) -- (14.7,3.9) -- (10.59,3.9) -- cycle; 
\draw[rounded corners, line width=1pt, opacity=0.3] (8.895,0.2) -- (14.7,0.2) -- (14.7,3.9) -- (8.895,3.9) -- cycle; 
\draw[rounded corners, line width=1pt, opacity=0.3] (8.895,0.1) -- (15.9,0.1) -- (15.9,3.9) -- (8.895,3.9) -- cycle; 
\draw[rounded corners, line width=1pt, opacity=0.3] (0.1,0.1) -- (8.2,0.1) -- (8.2,3.9) -- (0.1,3.9) -- cycle; 
\Vertex[Pseudo,x=8.1,y=2,size=0]{left};
\Vertex[Pseudo,x=8.995,y=2,size=0]{right};
\Vertex[Pseudo,x=1.42,y=0.7,size=0]{b0blow1};
\Vertex[Pseudo,x=2.40,y=1.95,size=0]{b0blow2};
\Vertex[Pseudo,x=3.1,y=2.6,size=0]{cbbsstart};
\Vertex[Pseudo,x=5.2,y=2,size=0]{bhighbright};
\Vertex[Pseudo,x=5.2,y=1.52,size=0]{cbhighbright};
\Vertex[Pseudo,x=10.630,y=2.4,size=0]{atin};
\Vertex[Pseudo,x=11.74,y=3.352,size=0]{gotoin};
\Vertex[Pseudo,x=12.085,y=1.78,size=0]{photoin};
\Vertex[Pseudo,x=14.65,y=2.7,size=0]{tocirc};
\Vertex[x=1.920,y=1.603,size=0.8,shape=ellipse,style={minimum height=0.6cm,minimum width=1cm},opacity=0,fontscale=1.2,IdAsLabel]{bLow}
\Vertex[x=0.525,y=0.594,shape=ellipse,style={minimum height=0.6cm,minimum width=0.6cm},opacity=0,fontscale=1.2,IdAsLabel]{b0}
\Vertex[x=4.810,y=2.590,size=1,shape=ellipse,style={minimum height=0.6cm,minimum width=1cm},opacity=0,fontscale=1.2,IdAsLabel]{bHigh}
\Vertex[x=6.605,y=2.590,size=1,shape=ellipse,style={minimum height=0.6cm,minimum width=1cm},opacity=0,fontscale=1.2,IdAsLabel]{bright}
\Vertex[x=3.715,y=0.550,size=0.8,shape=ellipse,style={minimum height=0.6cm,minimum width=1cm},opacity=0,fontscale=1.2,IdAsLabel]{Land}
\Vertex[x=3.715,y=1.603,size=0.8,shape=ellipse,style={minimum height=0.6cm,minimum width=1cm},opacity=0,fontscale=1.2,IdAsLabel]{calm}
\Vertex[x=7.500,y=1.203,size=1,shape=ellipse,style={minimum height=0.6cm,minimum width=1cm},opacity=0,fontscale=1.2,label=Avoid]{Safety}
\Vertex[x=9.495,y=1.435,size=1,shape=ellipse,style={minimum height=0.6cm,minimum width=1cm},opacity=0,fontscale=1.2,label=goal]{target}
\Vertex[x=10.990,y=2.2,opacity=0,fontscale=1.2,IdAsLabel]{at}
\Vertex[x=12.285,y=1,size=1,shape=ellipse,style={minimum height=0.6cm,minimum width=1.8cm},opacity=0,fontscale=1.2,label=Photograph]{Photo}
\Vertex[x=13.480,y=2.1,size=1,shape=ellipse,style={minimum height=0.6cm,minimum width=1.4cm},opacity=0,fontscale=1.2,label=Descend]{Down}
\Vertex[x=12.285,y=3.352,size=1,shape=ellipse,style={minimum height=0.6cm,minimum width=1cm},opacity=0,fontscale=1.2,IdAsLabel]{GoTo}
\Vertex[x=14.080,y=3.352,shape=ellipse,style={minimum height=0.6cm,minimum width=1.2cm},opacity=0,fontscale=1.2,label=Ascend]{Up}
\Vertex[x=15.315,y=1.5,opacity=0,size=1,shape=ellipse,style={minimum height=0.6cm,minimum width=1cm},fontscale=1.2,IdAsLabel]{Circle}
\Edge[,label=$\bm{s}$,fontscale=1.2,Direct](left)(right)
\Edge[,label=$\bm{s}$,fontscale=1.2,Direct](b0blow1)(Land)
\Edge[,label=$\bm{f}$,fontscale=1.2,Direct](b0blow2)(cbbsstart)
\Edge[,label=$\bm{f}$,fontscale=1.2,Direct](b0)(bLow)
\Edge[,label=$\bm{f}$,fontscale=1.2,Direct](bHigh)(bright)
\Edge[,label=$\bm{s}$,fontscale=1.2,Direct](calm)(bhighbright)
\Edge[,label=$\bm{f}$,fontscale=1.2,Direct](cbhighbright)(Safety)
\Edge[,label=$\bm{s}$,fontscale=1.2,Direct](target)(atin)
\Edge[,label=$\bm{s}$,fontscale=1.2,Direct](at)(photoin)
\Edge[,label=$\bm{f}$,fontscale=1.2,Direct](at)(gotoin)
\Edge[,label=$\bm{m}$,fontscale=1.2,Direct](Photo)(Down)
\Edge[,label=$\bm{m}$,fontscale=1.2,Direct](GoTo)(Up)
\Edge[,label=$\bm{f}$,fontscale=1.2,Direct](tocirc)(Circle)
\end{tikzpicture}
    \caption{(\emph{Above}:) A \emph{decision structure} for control of a solar-powered drone, used in the example in Section \ref{sec:example}. Decision structures are control architectures which generalise BTs, TRs, DTs and $k$-BTs. (\emph{Below}:) Its unique \emph{module decomposition}. Each circled subset of the nodes is what we call a \emph{module}. We show how this concept captures the intuitive notion of modularity, and corresponds to subtrees in BTs, $k$-BTs and DTs. This structure has \emph{essential complexity} 2 (see Section \ref{sec:complexitymeasure}).}
    \label{fig:motivation}
\end{figure}
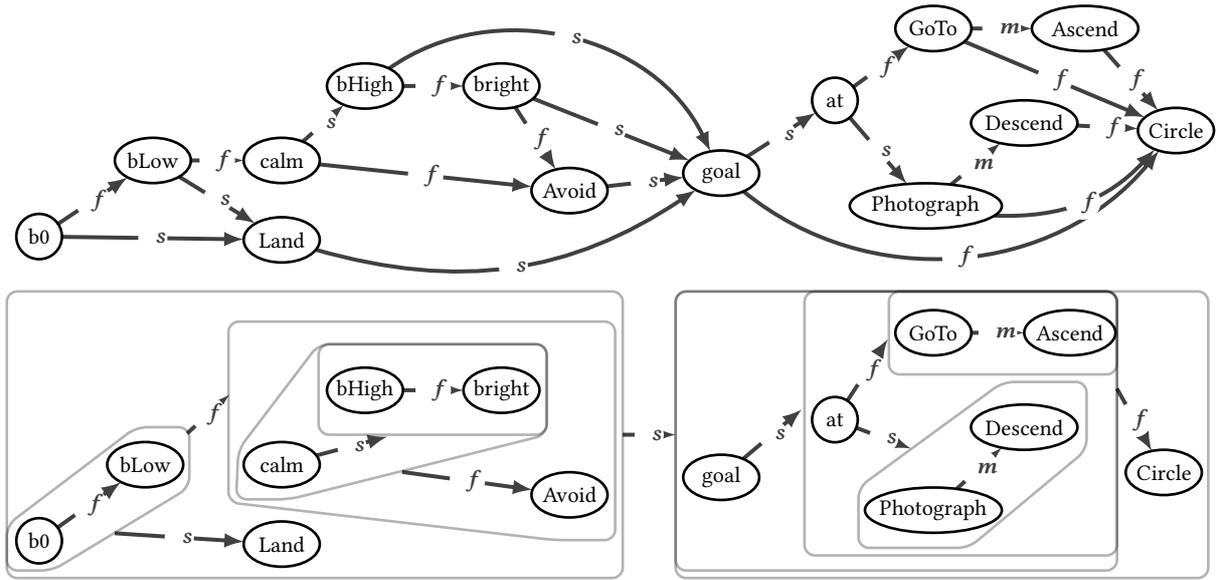

The concept of modularity is a fundamental principle of engineering, occurring in both cyber and physical systems. Modularity allows systems to be decomposed into simple and well-defined subparts, or \emph{modules}, which can be modified, reused or rearranged to form new structures. In this paper, we examine modularity in control architectures from a theoretical and practical perspective, taking inspiration from graph theory and structured and object-oriented programming.

Designing correct and intelligent behavior for autonomous agents is a classic problem in AI. In particular, \emph{control architectures} are a tool for constructing agents which react correctly to their environment. Control architectures encapsulate a strategy for deciding on actions while considering the state of the system and its environment. Architectures have been realised in a variety of forms, rising and falling in popularity in various niches of control and AI. These include Finite State Machines~\cite{hopcroft2001introduction}, Nilsson's Teleo-Reactive programs (TRs)~\cite{nilsson1993teleo}, Decision Trees and more recently Behavior Trees (BTs)~\cite{btbook} and generalised Behavior Trees ($k$-BTs)~\cite{kbts}. Each have strengths and weaknesses in various contexts, balancing out differing desirable traits such as modularity, simplicity, readability, expressiveness, reactiveness and richness of theory. As autonomous systems grow in complexity, modularity becomes an ever more important property in control architectures. However, lacking a formal definition, the concept is generally applied informally as an engineering principle. Cyber-physical systems are often more complex and have significantly higher costs of failure than entirely virtual AI systems. Therefore, there is increased burden on AI to be reliable and transparent in their use of modularity. We believe it is not enough to argue that a control architecture is modular---we must know \emph{why} it is modular, and understand fully the consequences for design, testing and verification.

The Oxford Dictionary defines a \emph{module} in its computing sense as ``any of a number of distinct but interrelated units from which a program may be built up or into which a complex activity may be analyzed." Building from this, and borrowing intuition from software design, we consider a `modular system' as one which is composed of individually simple and analysable \emph{modules}, where the whole can be constructed by composing the modules. Modules encapsulate data or policies, with the relationships between modules enforced through fixed interfaces, allowing modification to occur without cascading effects.

These ideas have played central roles in the development of programming and AI. In 1968 Dijkstra~\cite{gotoharmful} argued unstructured use of the \texttt{goto} statement was harmful and introduced errors to software. This formed the basis of the structured programming movement, which suggested that programs be composed of a small number of fixed constructs: sequencing, selection and iteration~\cite{dahl1972structured,dijkstra1970notes}. From the perspective of the \emph{control-flow graph}, which is a graph representation of the block structure of a program, this corresponds to building the graph recursively by combining subgraphs corresponding to each of these constructs~\cite{kosaraju1974analysis}. Later, McCabe~\cite{mccabe1976complexity} provided a complexity measure applied to the control-flow graph, called \emph{cyclomatic complexity}, which measured the difficulty of testing a program or function. This metric is now standard in software engineering~\cite{watson1996structured}. He then showed, similarly to Dijkstra, that control-flow graphs could be `decomposed' into proper subgraphs with single entry and exit nodes, with the graphs of structured programs being precisely those decomposable into a structure of complexity 1. The result was a measure which indicated the degree of `unstructuredness', or lack of modularity. Separately and almost concurrently, Gallai~\cite{gallai1967transitiv} developed the concept of a \emph{module} in graph theory, which can be thought of as a subgraph which presents a fixed interface to all other nodes of a graph. Moreover, the four key principles of object-oriented programming (encapsulation, abstraction, inheritance and polymorphism) all revolve around the concept of hiding information behind fixed interfaces. Modularity is also an important concept in AI, where modular control architectures such as Teleo-Reactive Programs and Behavior Trees are common. BTs first appeared in game AI~\cite{isla2015handling}, where constructing intelligent and reusable behavior has become critical. More recently, BTs are gaining significant popularity in robotics~\cite{btbook,surveyofbts,integratedARM}, with advocates arguing this is due to their reactiveness and modularity. BTs are considered modular because subtrees and individual nodes present a fixed interface, given by three \emph{return values}~\cite{btbook}. Hence subtrees of a BT are themselves BTs, and thus reusable and readable.

In this paper we propose a formal definition for `modules' in reactive control architectures (such as BTs, TRs, DTs and $k$-BTs). \emph{Reactive} architectures~\cite{nilsson1993teleo,kbts} are those which select an action on the basis of the current input only, and not any past inputs. Given a control architecture we provide a procedure for identifying structures within such architectures as modules that are agreeable to one's intuition. Following~\cite{mccabe1976complexity}, we develop a complexity measure defining the degree of modularity of control architectures and show how these results allow us to compare, analyse and formally verify control architectures in a modular way.

 We present definitions of control architectures and actions and use these to formalise a number of reactive control architectures (BTs, $k$-BTs, DTs and TRs). We then define a new architecture from directed acyclic graphs which we call a \emph{decision structure}, and show how all of the above architectures can be identified canonically with sets of decision structures. Following the graph-theoretic definition of~\cite{gallai1967transitiv}, we define \emph{decision structure modules}, which are subgraphs of decision structures grouped by their interfaces. We show how this provides a unique decomposition of any decision structure into modules which can be found in $O(n^2k)$ time, where $n$ is the number of nodes and $k$ the number of arc labels. Figure~\ref{fig:motivation} gives an example of a decision structure and its module decomposition. We show that modules in decision structures obey many of the properties of modules in graph theory. Further, we show how the decision structures corresponding to BTs, TRs and $k$-BTs and DTs can be naturally characterised by their modules. These characterisations show clearly which kind of action selections \emph{cannot} be constructed in each architecture.

Additionally, we prove a result with important consequences for hierarchical formal verification. We prove that the `independence' of modules holds in a formal sense, in that in any verification scheme, modifying a module within a verified structure has no more of an effect than modifying a single action. We show how we can verify the correctness of changes to actions locally, without repeating the verification on their hosting architecture. This allows verification to be done hierarchically, by verifying simple designs which are then refined. This makes possible the creation of libraries of verified decision-making for complex tasks which can be easily recombined while preserving correctness.

To demonstrate the power of the above result concretely, we define a new verification scheme capable of formally verifying any reactive architecture with Linear Temporal Logic (LTL). This scheme generalises the recent work of~\cite{biggar2020framework}, which showed that the `modularity' of BTs allows subtrees to be modified and formally verified separately while preserving the correctness of the tree. Our result is consistent with this, as we show that the subtrees of a BT are exactly the modules of the decision structure equivalent to that BT. Our result extends this theorem to a much wider class of architectures.

Finally, we provide an extensive practical example. We show a complex decision structure for control of a high-altitude solar-powered drone, inspired by~\cite{klockner2016behavior}. We conduct a formal verification of this structure given models of each action, and then show how to replace individual modules and guarantee the correctness of the modified structure while verifying only the modified modules. We then construct the module decomposition of the resultant structure, and determine which control architectures can equivalently represent it.

In summary, the main contributions of our paper are the proposed formalisation of reactive action selection, the concept of decision structures as a unifying framework for reactive architectures, modules in decision structures, the characterisations of each of these architectures, the complexity measure of control architectures and the result that modular structures can be verified in a modular way.

The paper is organised as follows. Section~\ref{sec:relatedwork} describes the related work, and Section~\ref{sec:btsandltl} presents the standard definitions of BTs, $k$-BTs, DTs, TRs and LTL. In Section~\ref{sec:cas and asms} we define control architectures and action selection. Section~\ref{sec:decisionstructures} introduces decision structures. Section~\ref{sec:modularity} focuses on modularity, defining modules and showing how decision structures can be decomposed into them. This is then used to define and prove the characterisations of these architectures. Section~\ref{sec:verrification} discusses formal verification and its relationship to modularity. Section~\ref{sec:example} goes through an example of the ideas introduced, and Section~\ref{sec:conclusions} concludes the paper. Few proofs of Theorems and Lemmas are found in the main text of the paper; most are instead in the Appendix.

\section{Related Work and Contributions} \label{sec:relatedwork}

This work investigates modularity and its consequences for reactive control structures. The ideas within draw on numerous sources, but we believe no other work has approached this precise question from this viewpoint.

From a broad perspective, this paper parallels developments in structured programming, classifying \emph{control-flow graphs} corresponding to structured programs~\cite{dahl1972structured}, as discussed in the introduction. Previous authors in the BT literature have used structured programming as a source of analogies. These papers~\cite{surveyofbts,btbook,ogren2012increasing,generalise} cite Dijkstra's paper `Goto statement considered harmful'~\cite{gotoharmful}, and use this as an analogy for the usefulness of modularity. Specifically, the tree structure of BTs is compared to the structure of a function call in structured programming, and contrasted against Finite State Machines (FSMs) whose unconstrained structure is similar to the use of the \texttt{goto} statement. This idea is examined more formally in~\cite{expressiveness}, comparing the result of~\cite{ogren2012increasing} to the B\"ohm-Jacopini Theorem~\cite{bohm1966flow} in structured programming. We take this farther by showing that decision structures corresponding to BTs can be recursively decomposed into minimally-complex modules, a result which closely resembles McCabe's result~\cite{mccabe1976complexity} showing that the control-flow graphs of structured programs can be recursively decomposed into minimally-complex graphs. Likewise, our definition of \emph{structural equivalence}, parallels the definition of \emph{weak reducibility} between control-flow graphs from~\cite{kosaraju1974analysis}. These connections have not been explored before, to the best of our knowledge.

The systems-theoretic principles behind our approach are influenced most by~\cite{willems2007behavioral,tabuada2009verification,lee2002embedded}. Our definitions of an action selection mechanism is similar to~\cite{expressiveness}, restricted to reactive action selection and with the addition of return values. This definition simplifies and generalises the formalism for BTs of~\cite{unifiedframework} by removing references to `ticks' and by interpreting control as an unspecified behavior map, following Lee~\cite{lee2002embedded}. This approach abstracts the continuous aspects to form a discrete problem, which as explained in Section~\ref{sec:decisionstructures}, also captures the discrete BT framework of~\cite{biggar2020framework}. This relies on some assumptions about reactiveness, for which we follow the arguments in~\cite{kbts}. Various authors in the fields of BTs and TRs have provided discussions of reactiveness, particularly~\cite{nilsson1993teleo,martens2018resourceful}. 

Prior to our construction of decision structures from BTs, some authors had performed similar translations of BTs to graphs. In~\cite{hannaford2019hidden}, a BT is translated to a graph, which is then used to construct a Hidden Markov Model. This graph is essentially a decision structure with the addition of `Success' and `Failure' nodes, and so can be directly translated into the representation here. However, this construction was not generalised, and the subsequent characterisation of BTs given was only a necessary and not a sufficient condition (see Remark~\ref{rem:hannaford}). We provide a clear approach to this translation consistent with the principles suggested in~\cite{kbts} and prove a characterisation of the BTs in these structures. Several papers have described Behavior Trees as Finite State Machines with special structure~\cite{isla2015handling,generalise,btbook}, with~\cite{generalise} providing an explicit translation from a BT to an (Hierarchical) FSM, in a manner similar to our construction of a decision structure. However, no efforts have been made to identify the class of FSMs to which these instances belong. FSMs are in general not reactive~\cite{expressiveness} as they make use of memory and internal state, which makes it difficult to compare them structurally to BTs. In this paper we use the decision structures as reactive analogues, and are able to determine the ``special structure" which provides their modularity. We are then able to argue which FSMs are BTs (under the translation from~\cite{generalise,btbook}, see Remark~\ref{rem:fsms}).

One of our results is that modifications to modules in decision structures can be verified as easily as modifications to actions. This result allows structure to be verified and have their correctness preserved by locally verifying the change to a module. A similar theorem is stated for BTs in Theorem 5.1 of~\cite{biggar2020framework}, where it shows that subtrees of BTs can be refined and verified locally. We show that the modules in decision structures corresponding to a BT are exactly the subtrees, and so this becomes a special case of our Theorem~\ref{verificationcorrespondence}, which generalises the result to all reactive architectures representable as decision structures, and to all verification schemes. The concrete verification scheme presented in Section~\ref{sec:verrification} generalises the method given in~\cite{biggar2020framework} and provides a simple and intuitive method for verifying any reactive architecture.
\section{Preliminaries} \label{sec:btsandltl}

\subsection{Mathematical notation}

 Given a function $f:X\to Y$, we denote the \emph{preimage} of a set $Z\subset Y$ as $f^{-1}(Z)$, and for $z\in Z$ write $f^{-1}(z)$ to mean $f^{-1}(\{z\})$. A \emph{graph} $G=(N,A)$ is a set of \emph{nodes} $N$ and \emph{arcs} $A$, where $A\subseteq N\times N$. Graphs in this paper are assumed to be directed. We use $N(G)$ and $A(G)$ to denote the node and arc sets. For an arc $(v_1,v_2)$, we call $v_2$ the \emph{head} and $v_1$ the \emph{tail}. A \emph{source} is a node that is the head of no arc and a \emph{sink} is the tail of no arc. A graph is labelled if there are maps $\ell:A(G)\to X$, $\eta:N(G)\to Y$ to some sets $X,Y$ of labels. A graph $K$ is a \emph{subgraph} of $G$, written $K\leq G$ if $N(K)\subseteq N(G)$ and $A(K) = \{ (v_1,v_2)\in A(G)\ |\ v_1,v_2\in N(K)\}$ (this is often called an \emph{induced subgraph}). If we are given a set of nodes $X\subset N(G)$, we will write the subgraph induced by $X$ in $G$ as $G[X]$. 
A sequence of nodes $v_1,\dots,v_n$ is a \emph{path} if for $i\in\{1,\dots,n-1\}$, $(v_i,v_{i+1})\in A(G)$. We say a node $q$ is an ancestor of a node $v$ if there is a path from $q$ to $v$. A path is a \emph{cycle} if it has length at least two and its first and last node are the same. A graph is acyclic if it has no cycles. A \emph{tree} is a graph where there is a node $r$, called the \emph{root}, such that there is exactly one path from $r$ to any other node (sometimes called an \emph{arborescence}). Trees are acyclic and the root is the only source. Sinks in trees are called \emph{leaves}, and we shall refer to other nodes in a tree as \emph{internal nodes}. Two labelled graphs $H$ and $G$ are \emph{isomorphic} if there exists a bijection $\Gamma: N(G) \to N(H)$ such that $(v_1,v_2)\in A(G)$ iff $(\Gamma(v_1),\Gamma(v_2))\in A(H)$ and $\ell(v_1,v_2) = \ell((\Gamma(v_1),\Gamma(v_2)))$.
A permutation is a bijection from a finite set to itself.

\subsection{Behavior Trees, Generalised Behavior Trees, Teleo-reactive programs and Decision Trees}
\subsubsection{Behavior Trees}
BTs are control architectures which take the form of trees. The execution of a BT occurs through signals called `ticks', which are generated by the root node and sent to its children. A node is executed when it receives ticks. Internal nodes tick their children when ticked, and are called \emph{control flow nodes} and leaf nodes are called \emph{execution nodes}. When ticked, each node can return one of three possible return values: `Success' if it has achieved its goal, `Failure' if it cannot operate properly and `Running' otherwise, indicating its execution is underway. Typically, there are four types of control flow nodes (Sequence, Fallback, Parallel, and Decorator) and two types of execution node (Action and Condition)~\cite{btbook}. Note that Fallback is sometimes called Selector. A Condition (drawn as an ellipse) checks some conditional statement, returning Success if true and Failure otherwise. An Action node (drawn as a rectangle) represents an action taken by the agent. Sequence nodes (drawn as a $\seq$ symbol) tick their children from left to right. If any children return Failure or Running that value is immediately returned by Sequence, and it returns Success only if every child returns Success. Fallback (drawn as $\fb$) is analogous to Sequence, except that it returns Failure only if every child returns Failure, and so on. The Parallel node (drawn as $\prlseq$) has a success threshold $M$, and ticks all of its $N$ children simultaneously, returning Success if $M$ of its children return Success, Failure if $N-M +1$ return Failure and Running otherwise. The Decorator node returns a value based on some user-defined policy regarding the return values of its children. For a more detailed discussion, we refer the reader to~\cite{btbook}. In this paper, we will not discuss the Parallel node, as the associated questions of concurrency are complex~\cite{improvingparallel} and tangential from the main results of this paper. Similarly we will also omit discussion of the Decorator node---its extreme flexibility makes formal analysis difficult and it can violate the readability and reactiveness properties of BTs~\cite{kbts}. Hence we will consider BTs to consist of leaf nodes which are Actions and Conditions and two internal nodes, Sequence and Fallback. This approach is taken in some other papers reasoning about BTs such as~\cite{biggar2020framework,martens2018resourceful}. Following~\cite{biggar2020framework}, we will often write BTs in infix notation, such as $A\seq (B \fb C) \seq D$ where the control flow nodes are interpreted as associative operators over the leaf nodes, and the tree structure is the syntax tree.
\subsubsection{Generalised Behavior Trees}
With this interpretation of BTs, it is straightforward to define the generalised Behavior Trees, also called the $k$-BTs~\cite{kbts}. A $k$-BT is a tree which executes by ticks from the root. When ticked, nodes can return one of $k$ values, which we interpret as integers from 1 to $k$. The control flow nodes come in $k$ types, which we write as $*_1,*_2,\dots,*_k$. Each control flow node $*_i$ ticks its children from left to right. If any child returns a value other than $i$, the node immediately returns that value, and $*_i$ returns $i$ only if all its children return $i$. $k$-BTs do not have an explicit ``Running" value; any value other than $1,\dots,k$ is interpreted as Running. If we interpret the return value 1 as ``Success", 2 as ``Failure", as in~\cite{kbts}, we can observe that the definitions of $*_1$ and $*_2$ correspond exactly to the Sequence and Fallback nodes for BTs. In this sense, the 2-BTs are exactly the classical Behavior Trees as described above, where any other return values are interpreted as ``Running".
\subsubsection{Teleo-reactive programs}

Teleo-reactive programs (TRs)~\cite{nilsson1993teleo} are lists of condition-action rules.\begin{align*}
    k_1 &\to a_1 \\ 
    k_2 &\to a_2 \\
    &\vdots \\
    k_n &\to a_n 
\end{align*} These are executed by continuously scanning the list of conditions $k_i$ in order, and executing the action $a$ associated with the first satisfied condition. If another condition higher in the list becomes true then the executing action switches immediately. The \emph{teleo} indicates that such lists are goal-oriented while \emph{reactive} is intended to describe how they react constantly to changes in the environment. The Teleo-reactive programs are equivalent to the 1-BTs~\cite{kbts}.
\subsubsection{Decision Trees}
Decision Trees (DTs) are decision-making tools which resemble the structure of if-then statements. Specifically, DTs are binary trees where leaves represent actions and internal nodes represent predicates. The two arcs out of each predicate are labelled by `True' and `False'. Execution of DTs occurs by beginning at the root and evaluating each predicate on the current input state until a leaf is reached, at which point that action is executed. At a predicate node, if it is true in the current input state the execution proceeds down the `True' arc, and otherwise down the `False' arc. Like TRs, DTs are executed by continuously checking the predicates against the current state of the world.

\subsection{Linear Temporal Logic}

Linear Temporal Logic (LTL)~\cite{ltl} is an extension of propositional logic which includes qualifiers over linear paths in time. It has long been used as a tool in formal verification of programs and systems. Given a set $AP$ of atomic propositions, the syntax of LTL is given by the following grammar~\cite{modelchecking}: 
\[
\varphi ::= p\ |\ \neg p\ |\ \varphi \lor \varphi\ |\ \X\varphi\ |\ \varphi\ \mathcal{U} \varphi
\]
where $p\in AP$. The temporal operators are \emph{next} $\X a$, which indicates $a$ is true in the subsequent state, and \emph{until} $a \mathcal{U} b$ indicating $a$ is true until a state where $b$ is true. We can derive other common operators from these as follows: \emph{and} $\varphi\land \psi :=\neg(\neg\varphi \lor\neg\psi)$, True $:= p\lor\neg p$, False $:=\neg$ True, \emph{implies} $\varphi \Rightarrow \psi := \neg \varphi \lor \psi$, \emph{eventually} $\event \varphi := \text{True}\ \mathcal{U} \varphi$ and always $\always \varphi := \neg\event\neg\varphi$. An LTL formula is satisfied by a sequence of truth assignments over $AP$. If $\phi$ and $\psi$ are LTL formulas, then $\phi\models \psi$ if for every sequence of states $\sigma$ where $\phi$ is satisfied, $\psi$ is also satisfied in $\sigma$ and we say $\phi$ entails $\psi$. We refer the reader to~\cite{modelchecking} for a more thorough discussion of LTL. (LTL is only needed for Sections~\ref{sec:verrification} and \ref{sec:example}.)

\section{Action Selection Mechanisms} \label{sec:cas and asms}

In this section we formalise \emph{action selection} by presenting definitions of `actions' and `selection', along with an abstraction of the `world' in which this process takes place. We aim for definitions applicable to a broad range of objects, so we will state them in generality, borrowing ideas from hybrid systems theory, particularly Tabuada~\cite{tabuada2009verification} and Willems~\cite{willems2007behavioral}.

 Let $\signalspace$ be a set we call the \emph{signal space}, and $\world$ a set we call the \emph{state space}. An element $x\in\signalspace$ we call a \emph{signal}, and an element $w\in\world$ we call a \emph{state}. We begin with a \emph{system} $\mathcal{W}$ (as in~\cite{tabuada2009verification}) we call the \emph{world}, where the inputs to this system is the set $\signalspace$ and the outputs are states in $\world$. In general we will assume $\mathcal{W}$ is non-deterministic.
\begin{defn}
A \emph{behavior} $\mathcal{B}$ is a map $\world\to\signalspace$, that is, a function taking a state in the state space and producing a corresponding signal in the signal space.
\end{defn}
Behaviors are the fundamental unit of our approach to control architectures. A behavior is an abstraction for a model of a controller, which reads some state as input from the world and produces some signal which in turn influences the evolution of that world, as in Figure \ref{fig:controldiagram}. We describe this by saying $\mathcal{B}$ \emph{acts on} $\mathcal{W}$. Within this definition, we allow a behavior to be arbitrarily complex. This definition is generic enough to allow for many models of computation (as in~\cite{lee2002embedded}) appropriate for embedded and cyber-physical systems.
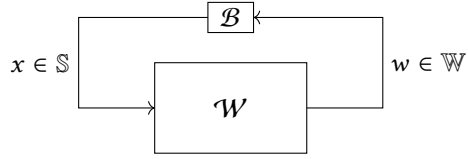
\begin{figure} 
    \centering
    \begin{tikzpicture}
    \draw (1.5,1.2) rectangle (3.5,0);
    \node at (2.5,0.6) {$\mathcal{W}$};
    \draw [->] (3.5,0.6) -- (4.5,0.6) -- 
    (4.5,1.8) -- (2.8,1.8);
    \draw [->] (2.2,1.8) -- (0.5,1.8) -- (0.5,0.6) -- (1.5,0.6);
    \draw (2.2,2) rectangle (2.8,1.6);
    \node at (2.5,1.8) {$\mathcal{B}$};
    \node at (5.1,1.2) {$w\in\world$};
    \node at (0,1.2) {$x\in\signalspace$};
    \end{tikzpicture}
    \caption{Model of the interaction of a behavior with the world.}
    \label{fig:controldiagram}
\end{figure}
\begin{defn}
An \emph{action} $\alpha$ is a pair $(\alpha_B, \alpha_R)$ where $\alpha_B$ is a behavior and $\alpha_R:\world\to \returnvals$ is a function we call the \emph{return value function}, where $\mathcal{R}$ is some fixed set of symbols we call \emph{return values}. 
We use $\Act$ to denote the set of all actions in $\mathcal{W}$.
\end{defn}

The \emph{return value function} $\alpha_R$ provides metadata regarding the behavior $\alpha_B$. We will use this metadata to construct composite actions (Definition \ref{def:asm}), reflecting the operation of control architectures. To motivate this definition, consider an Action node $A$ in a Behavior Tree. This corresponds with what we define above as an \emph{action}. The signals produced by that node in a given state form the behavior $A_B$. For BTs, fix some return values $\bm{s},\bm{f}\in\returnvals$. Then, the return value $A_R$ is $\bm{s}$ or $\bm{f}$ precisely when the node $A$ returns Success or Failure respectively. This shows that this definition is able to encode the BT concept of Action nodes. Throughout this paper, we shall write return values in boldface.

\begin{defn} \label{def:asm}
Let $\Delta\subseteq \Act$ be a finite set of actions. An \emph{action selection mechanism} (ASM) is a map $M:\world\to \Delta$, where there exists a finite set $X\subseteq \returnvals$ such that for all states $w_1,w_2\in\world$, if all actions $\alpha\in\Delta$ agree on the return values $\bm{r}\in X$, then $M(w_1) = M(w_2)$. That is, $\forall w_1,w_2\in\world$, $(\forall\alpha\in\Delta,\forall r\in X,\ (\alpha_R(w_1)=\bm{r} \Leftrightarrow \alpha_R(w_2) = \bm{r})) \implies M(w_1) = M(w_2)$. Each ASM $M$ has a derived action $(M_B,M_R)$ where $M_B(w) = M(w)_B(w)$, $M_R(w) = M(w)_R(w)$. We will write $\asms$ for the set of ASMs in $\mathcal{W}$.
\end{defn}

In essence, an ASM is a map from states to actions which is based on the value of a finite number of Boolean variables. These variables are exactly the predicates `$\alpha_R(w) = \bm{r}$' with $\alpha\in\Delta$ and $\bm{r}\in X$. While the set $\returnvals$ is not assumed to be finite, we assume any given ASM uses only a finite subset $X$, which enforces that there are only a finite number of variables overall.

An action selection mechanism is a tool for aggregating individual behaviors into a larger behavior, based on a finite set of their return values. This is best viewed in the context of Actor-oriented design, as in~\cite{lee2002embedded}, where actions are the actors, with the return values describing the information shared with the port by which they interface with other actors. A specific action selection mechanism fixes a model of computation for the system, assigning a meaning to individual components and to the links between them as mediated by their ports.

The overall concept of action selection is very general. The definition of action selection given here is intended to model how BTs, $k$-BTs, TRs and DTs work, and we will give examples of how these definitions correspond in Section~\ref{sec:instances}. One possible extension of this concept is to allow ASMs to access some \emph{memory}, such as in Finite State Machines, which also select actions but do so on the basis of both the input state and the previous action selected. We focus here on `reactive action selection' because it allows for convenient discussions of modularity via \emph{decision structures} (introduced in Section~\ref{sec:decisionstructures}). See also Section~\ref{reactiveness} for a discussion of reactiveness.

\subsection{Instances of ASMs} \label{sec:instances}

Behavior Trees (BTs), generalised Behavior Trees ($k$-BTs), Teleo-reactive programs (TRs) and Decision Trees (DTs) can all be formalised as Action Selection Mechanisms. 

We start with BTs. To begin, fix some values $\bm{s},\bm{f}\in\returnvals$, to represent `Success' and `Failure'. All other return values in $\returnvals$ will be ignored in the architecture, as so are treated as the classical `Running' return value. We do not assign `Running' itself to a value, because this prevents actions from potentially returning more values which are not handled by BTs (see Remark~\ref{rem:bt return}). Then consider any Behavior Tree. The associated ASM is defined through the structure and semantics of the BT. Specifically, the ASM is constructed inductively by the following rules, where $a$ and $b$ are actions:
\begin{align*}
    \forall w\in\world,\ (a\seq b)(w) = \begin{cases} b, & a_R(w) = \bm{s} \\ a, &\text{otherwise}\end{cases} &\qquad
    (a\fb b)(w) = \begin{cases} b, & a_R(w) = \bm{f} \\ a, &\text{otherwise}\end{cases}
\end{align*}
The following is a concrete example.
\begin{exmp}\label{ex:BTasCA}
Consider the BT $a \fb (b \seq c)$. We obtain an ASM $M$:
\begin{equation*}
\forall w\in\world,  M(w) = \begin{cases}
a, & a_R(w) \neq \bm{f} \\
    b, & b_R(w)\neq \bm{s}\land a_R(w) = \bm{f} \\
    c, & b_R(w) = \bm{s} \land a_R(w) = \bm{f}
    \end{cases}
\end{equation*}

Conceptually, each action is selected if it would be the last leaf ticked in an execution of that tree. Hence the above translates to ``select $a$ unless $a$ returns Failure, in which case select $b$ unless $b$ returns Success, in which case select $c$."
\end{exmp}
\begin{remark}[Return values in Behavior Trees] \label{rem:bt return}
Those familiar with BTs will note at this point that an action can be selected by the BT even if it returns Success or Failure---for instance, in the previous example, $a$ is selected whenever it returns Success. Traditionally in a BT, Success indicates that the action need not be selected, and actions are selected only if they return `Running'. This information is still captured here, because the return value of $M$'s derived action is the same as the return value of $a$. That is, even though $a$ is selected, the whole tree still returned Success. This allows this tree (and hence $M$'s derived action) to be nested within another tree (or ASM) and still behave as expected, an idea which is key to modularity. It is for similar reasons that we interpret `Running' as any other return value. This allows actions to be defined with potentially numerous return values, which if not handled within a given ASM are ignored~\cite{kbts}. For instance, an action returning three values $\bm{s},\bm{f}$ and $\bm{m}$ may act as if it returns only $\bm{s}$ and $\bm{f}$ within a BT, but if this BT is used as a subtree of a $k$-BT which handles the value $\bm{m}$ then this return value becomes relevant. See Section~\ref{sec:example} for examples of such actions.
\end{remark}
The $k$-BTs can be constructed in the same way, by composing the ASM from the $k$ operators. Fix some $k$ values $\{\bm{1},\dots,\bm{k}\}\subseteq \returnvals$, with the return value $i$ corresponding to the operator $*_i$. A $k$-BT is again an ASM where the action selected is that which labels the final leaf ticked in a traversal of the tree.

The argument is similar for Teleo-reactive programs.  Given a TR as a list of preconditions and actions, we fix some values $d_i\in\returnvals$ to represent the negation of the $i$th precondition. The ASM then selects the top of the list unless its return value is $d_1$, in which case it checks the second item in the list unless its return value is $d_2$, and so on. We negate the precondition to make TRs more consistent with BTs, which use return values to indicate that the action should \emph{not} be selected.

For DTs, as for BTs, we interpret predicates/conditions to be actions which always return one of exactly two return values, which we interpret as True and False (or Success and Failure in the BT case~\cite{btbook,biggar2020framework}). Fix values $\bm{\top},\bm{\bot}\in\returnvals$ to represent a predicate being True or False. The ASM is defined by the usual DT semantics. Given a state $w\in \world$, at each internal node labelled by action $\alpha$, proceed down the `True' arc if $\alpha_R(w) = \bm{\top}$ and proceed down the `False' arc if $\alpha_R(w) = \bm{\bot}$, until a leaf is reached which is then selected. Note that although we needed to assume that predicates returned precisely two values, the actions labelling leaves can return any number of values. While these values are not handled by DTs, if this ASM were nested within another one, they could be handled as suggested in Remark~\ref{rem:bt return}. This will become critical once we have translated these architectures to decision structures.

\begin{remark}[A note on return values]
In explaining how each of these architectures are interpreted as ASMs, we chose a number of `fixed' return value symbols, such as $\bm{s},\bm{f},\bm{d},\bm{1},\bm{2},\bm{\top},\bm{\bot}\dots\in\returnvals$. While we associated each of these with corresponding concepts for each architecture, such as $\bm{s}$ representing Success, it is important to recall that these have no intrinsic meaning, nor are they necessarily distinct values. As an example, the pair of constructions corresponding to the Sequence and Fallback operators for BTs made use of $\bm{s}$ and $\bm{f}$, but would have made sense for any fixed pair of values. This means that any ASM for which we can choose some return values which allows us to interpret it as a BT we will consider to be a BT, and we will refer to its return values as $\bm{s}$ and $\bm{f}$. As another example, we generally assume $\bm{s} = \bm{1}$ and $\bm{f} = \bm{2}$ when comparing BTs and $k$-BTs, or similarly $\bm{d} = \bm{f}$ to compare BTs and TRs, as in~\cite{generalise}. In the next section we will formalise this idea. Later we will compare decision structures which are labelled by return values, and there it will be useful to consider two structures essentially the same if we can translate between the return values used as labels.
\end{remark}
\subsection{Equivalence of ASMs}
So far, ASMs have depended crucially on the actions from which they select. However, modularity is a structural property of decision-making. Thus, we need an essentially `action-independent' notion of equivalence between architectures, which depends only on the structure.
\begin{defn} \label{def:equivalence}
Consider a reactive control architecture, such as a BT, $k$-BT, DT or TR, where the actions are labelled on `nodes'. For BTs and $k$-BTs, these `nodes' are the leaf nodes, for TRs these are the items in the list, and for DTs (and later also decision structures) these are all nodes in the tree. Any two architectures $X$ and $Y$ are structurally equivalent if they have the same number ($n$) of nodes and there exist orderings of both node sets such that labelling the nodes of both structures in those orders by any $n$ actions gives the same ASM for both architectures.
\end{defn}
This equivalence removes dependence on the actions, and is consistent with our intuition regarding `structure'. For instance, any two BTs (such as $a\fb(b\seq c)$ and $d\fb(e\seq f)$ as considered in Example~\ref{ex:BTasCA}) which have the same structure as unlabelled trees are structurally equivalent. This equivalence goes further than this, however. Some BTs are different but semantically identical, such as $a\seq (b\seq c)$ and $(a\seq b)\seq c$, and these are structurally equivalent, as they always give the same ASM when labelled in this order. A more subtle point to note is that the trees $a\seq (b\fb c)$ and $(a\seq b)\fb (a\seq c)$, though equal as ASMs for any fixed $a,b,c\in\Act$, are \emph{not} structurally equivalent because in general the second structure $(\circ_1\seq \circ_2)\fb (\circ_3\seq \circ_4)$ is labelled by four actions, and so the ASMs are only equal in the special case where the actions labelling $\circ_1$ and $\circ_3$ are the same. Essentially, the equality between these trees cannot be determined without specifying the input actions, while the structural equivalence $a\seq (b\seq c)\equiv (a\seq b)\seq c$ can be determined without instantiating the tree. We omit the proof, but one can show that each $k$-BT is structurally equivalent to a unique $k$-BT in `\emph{compressed form}', where all of the children of a $*_i$ node are either leaves or are not themselves $*_i$ nodes. We reach this form by merging any parent-child pairs labelled by the same operator. While we have focused on BTs in this discussion, the same principles hold for other ASMs.

\subsection{Reactiveness} \label{reactiveness}

In arguing that the aforementioned structures are ASMs, we have implicitly argued that they select actions `memorylessly', based only on the knowledge of the current state. This property is called \emph{reactiveness}~\cite{btbook}. In general, this is how these architectures are interpreted~\cite{btbook,nilsson1993teleo,biggar2020framework,martens2018resourceful,kbts}. However, this subtly relies on some assumptions about the interpretations of these architectures, as discussed in~\cite{martens2018resourceful,kbts,expressiveness}. The analytical benefit of a `purely' reactive ASM $M$ is that if the state of the world is $w\in\world$, then the action currently selected is definitely $M(w)$. This `reactiveness property' allows us to more easily verify whether the architecture reacts appropriately to changes in the world state. In our paper we assume these architectures are event-triggered, where any relevant change in the world system $\mathcal{W}$ generates an output state $w\in\world$ which is considered as the event. If we assume (as in~\cite{biggar2020framework,surveyofbts}) that the ticks are frequent compared to the world system time scale, BTs can be fairly considered to be event-triggered.
\begin{asmp}
We assume that computing the output of an ASM takes negligible time on the world time scale.
\end{asmp}
This ensures the world state does not change between when the ASM begins selecting the action and when it completes the selection, allowing us to interpret them as reactive.

\section{Comparing control architectures using decision structures} \label{sec:decisionstructures}

It is difficult to make general statements about control architectures because it is difficult to compare architectures that are presented in stylistically different manners, such as BTs and DTs. Here we introduce a new type of control architecture, the \emph{decision structure}, which subsumes many other classes of architectures and provides us with a framework to examine relationships among them. Later, decision structures allow us to formalise the property of modularity.

\begin{defn}\label{def:ASM-DG}
A \emph{decision structure} is an labelled acyclic graph $Z = (N,A,\ell:A\to \returnvals, \eta:N\to\Act)$ with a unique source, where all arcs out of a given node have distinct labels. That is, for any $(q,v_1),(q,v_2)\in A$, $v_1\neq v_2 \implies \ell(q,v_1) \neq \ell(q,v_2)$. If there is an arc labelled $\bm{r}$ out of a node $v$, we shall refer to it uniquely as the \emph{$\bm{r}$ arc out of $v$}. We will denote the set of all such graphs by $\Swt$.
\end{defn}
A decision structure $Z$ is an ASM. We interpret it in this way as follows. 
For any state $w\in\world$,
\begin{itemize}
    \item Begin at the source.
    \item Let $\alpha$ be the action labelling the current node, and suppose $\alpha_R(w)=\bm{r}\in\returnvals$. If the $\bm{r}$ arc exists, go to the head of the $\bm{r}$ arc and repeat.
    \item Otherwise, select $\alpha$.
\end{itemize}
This selects an action for any state $w\in\world$, based only on Boolean combinations of a finite number of return values (the labels on the arcs), so this is an ASM.

The execution model of decision structures as an ASM is superficially similar to that of an FSM. The selection process begins at the initial node, and proceeds along arcs based on the return values of each node in the input state until it reaches a node where none of the return values match outgoing arcs, at which point it selects the action labelling that node. This is reactive because, unlike FSMs, for every selection we begin at the source, independent of which action was selected at the previous time step.

\begin{defn}
Given a set $T$ of architectures, such as the set of BTs or TRs, we call a map $\kappa_T:T\to \Swt$ a \emph{construction map} if for any $x\in T$, $\kappa_T(x)$ is structurally equivalent to $x$ (Definition \ref{def:equivalence}). We call a set $T$ of architectures \emph{realisable} if there exists a construction map $\kappa_T:T\to \Swt$.
\end{defn}
\begin{lemma} \label{isomorphism of digraphs}
Two decision structures are structurally equivalent if and only if they are isomorphic as labelled graphs.
\end{lemma}
\begin{corollary} \label{uniqueconstructionmap}
For any set $T$ of architectures, if there exists a construction map $\kappa_T:T\to\Swt$ then it is unique.
\end{corollary}

We now know that if it is possible to express a given set of architectures as decision structures, then that representation is unique. This result is useful because the standard presentations of many sets of architectures are \emph{not} unique; as was already discussed, there are distinct BTs which are structurally equivalent. We will use this result to compare architectures that are otherwise superficially different, such as BTs and TRs. To do this though, we must show they are in fact realisable.

\begin{lemma} \label{bt_construction_map}
For BTs, the unique construction map $\kappa_{BT}$ is given by the following procedure:
\begin{itemize}
    \item Let $A_1,\dots, A_n$ be the action nodes read left to right in the BT.
    \item For each action $A_i$, construct a node $v_i$ in the decision structure, labelled by $A_i$.
    \item For each node $v_i$ labelled by $A_i$, if $A_j$ is the action subsequently ticked when $A_i$ returns Success/Failure, then send an arc labelled $\bm{s}$/$\bm{f}$ from $v_i$ to $v_j$.
\end{itemize}
\end{lemma}
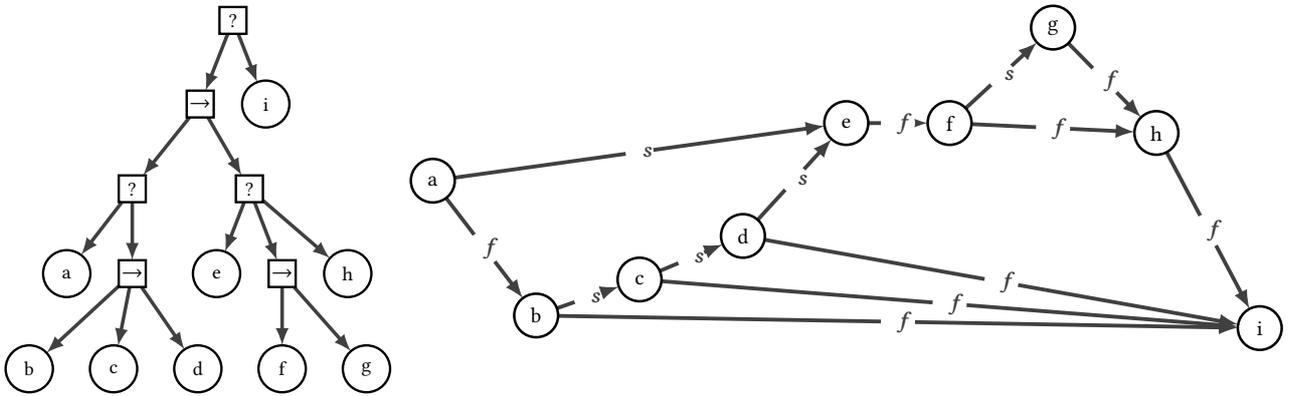
\begin{figure}[ht]
\centering
\caption{An example of a BT with its corresponding decision structure}
\label{fig:bt_ss}
\begin{tikzpicture}
\clip (0,0) rectangle (6.0,6.0);
\Vertex[x=3.520,y=5.600,size=0.4,opacity=0,label=?,fontscale=1.2,shape=rectangle]{1}
\Vertex[x=3.035,y=4.352,size=0.4,opacity=0,label=$\rightarrow$,fontscale=1.2,shape=rectangle]{2}
\Vertex[x=2.029,y=3.087,size=0.4,opacity=0,label=?,fontscale=1.2,shape=rectangle]{3}
\Vertex[x=1.059,y=1.821,size=0.7,opacity=0,label=a,fontscale=1.2,shape=ellipse]{4}
\Vertex[x=2.029,y=1.821,size=0.4,opacity=0,label=$\rightarrow$,fontscale=1.2,shape=rectangle]{5}
\Vertex[x=0.504,y=0.400,size=0.7,opacity=0,label=b,fontscale=1.2,shape=ellipse]{6}
\Vertex[x=1.752,y=0.400,size=0.7,opacity=0,label=c,fontscale=1.2,shape=ellipse]{7}
\Vertex[x=3.000,y=0.400,size=0.7,opacity=0,label=d,fontscale=1.2,shape=ellipse]{8}
\Vertex[x=3.763,y=3.087,size=0.4,opacity=0,label=?,fontscale=1.2,shape=rectangle]{9}
\Vertex[x=3.277,y=1.821,size=0.7,opacity=0,label=e,fontscale=1.2,shape=ellipse]{10}
\Vertex[x=4.248,y=1.821,size=0.4,opacity=0,label=$\rightarrow$,fontscale=1.2,shape=rectangle]{11}
\Vertex[x=4.248,y=0.400,size=0.7,opacity=0,label=f,fontscale=1.2,shape=ellipse]{12}
\Vertex[x=5.496,y=0.400,size=0.7,opacity=0,label=g,fontscale=1.2,shape=ellipse]{13}
\Vertex[x=5.219,y=1.821,size=0.7,opacity=0,label=h,fontscale=1.2,shape=ellipse]{14}
\Vertex[x=4.005,y=4.352,size=0.7,opacity=0,label=i,fontscale=1.2,shape=ellipse]{15}
\Edge[,fontscale=1.2,Direct](1)(2)
\Edge[,fontscale=1.2,Direct](1)(15)
\Edge[,fontscale=1.2,Direct](2)(3)
\Edge[,fontscale=1.2,Direct](2)(9)
\Edge[,fontscale=1.2,Direct](3)(4)
\Edge[,fontscale=1.2,Direct](3)(5)
\Edge[,fontscale=1.2,Direct](5)(6)
\Edge[,fontscale=1.2,Direct](5)(7)
\Edge[,fontscale=1.2,Direct](5)(8)
\Edge[,fontscale=1.2,Direct](9)(10)
\Edge[,fontscale=1.2,Direct](9)(11)
\Edge[,fontscale=1.2,Direct](9)(14)
\Edge[,fontscale=1.2,Direct](11)(12)
\Edge[,fontscale=1.2,Direct](11)(13)
\end{tikzpicture}
\begin{tikzpicture}
\clip (0,0) rectangle (12.0,6.0);
\Vertex[x=11.650,y=0.930,opacity=0,fontscale=1.2,IdAsLabel]{i}
\Vertex[x=10.238,y=3.619,opacity=0,fontscale=1.2,IdAsLabel]{h}
\Vertex[x=8.825,y=5.070,opacity=0,fontscale=1.2,IdAsLabel]{g}
\Vertex[x=7.412,y=3.755,opacity=0,fontscale=1.2,IdAsLabel]{f}
\Vertex[x=6.000,y=3.755,opacity=0,fontscale=1.2,IdAsLabel]{e}
\Vertex[x=4.588,y=2.2,opacity=0,fontscale=1.2,IdAsLabel]{d}
\Vertex[x=3.175,y=1.6,opacity=0,fontscale=1.2,IdAsLabel]{c}
\Vertex[x=1.762,y=1.104,opacity=0,fontscale=1.2,IdAsLabel]{b}
\Vertex[x=0.350,y=2.961,opacity=0,fontscale=1.2,IdAsLabel]{a}
\Edge[,label=$\bm{f}$,fontscale=1.2,Direct](h)(i)
\Edge[,label=$\bm{f}$,fontscale=1.2,Direct](g)(h)
\Edge[,label=$\bm{s}$,fontscale=1.2,Direct](f)(g)
\Edge[,label=$\bm{f}$,fontscale=1.2,Direct](f)(h)
\Edge[,label=$\bm{f}$,fontscale=1.2,Direct](e)(f)
\Edge[,label=$\bm{s}$,fontscale=1.2,Direct](d)(e)
\Edge[,label=$\bm{f}$,fontscale=1.2,Direct](d)(i)
\Edge[,label=$\bm{s}$,fontscale=1.2,Direct](c)(d)
\Edge[,label=$\bm{f}$,fontscale=1.2,Direct](c)(i)
\Edge[,label=$\bm{s}$,fontscale=1.2,Direct](b)(c)
\Edge[,label=$\bm{f}$,fontscale=1.2,Direct](b)(i)
\Edge[,label=$\bm{s}$,fontscale=1.2,Direct](a)(e)
\Edge[,label=$\bm{f}$,fontscale=1.2,Direct](a)(b)
\end{tikzpicture}
\end{figure}

Generalising the above construction, we derive the case for $k$-BTs.
\begin{lemma}
For $k$-BTs, the unique construction map $\kappa_{k\text{-BT}}$ is given by the following procedure:
\begin{itemize}
    \item Let $A_1,\dots, A_n$ be the action nodes read left to right.
    \item For each action, construct a node $v_i$ in the decision structure labelled by that action.
    \item For each node $v_i$ labelled by $A_i$, if $A_j$ is the action subsequently ticked when $A_i$ returns $1,\dots,k$, then send an arc labelled $1,\dots,k$ to $v_j$.
\end{itemize}
\end{lemma}

\begin{lemma}
The unique construction map $\kappa_{DT}$ for DTs is given by the identity map on the underlying tree with the arc labels out of predicates given by $\bm{\top}$ on the `Yes' arc and $\bm{\bot}$ on the `No' arc.
\end{lemma}
\begin{remark}[Decision trees as decision structures]
Having the identity as a construction map means that DTs are just decision structures with a particular form. The naming of decision structures is intended to highlight this, as their execution is essentially a generalisation of decision trees.
\end{remark} 

\begin{corollary}
The unique construction map $\kappa_{TR}$ for TRs is given by constructing a path graph with the nodes labelled by actions in the order given by the TR, with the single arc out of each node labelled $\bm{d}$.
\end{corollary}

Now we have unique ways of expressing these realisable architectures in the common execution model given by decision structures. We can compare the structures by the sets $BT$, $DT$, $k$-$BT$, $TR\subseteq\Swt$. Any ASM that belongs to multiple sets has an structurally equivalent interpretation in all classes of ASMs of which it is a member. We will now freely refer to `the' decision structure of an ASM, by which we mean its image under a construction map. 

However, we have still not justified why decision structures particularly are the `right' abstraction. Why translate BTs, DTs and TRs to decision structures, and not some other fairly general architecture? Essentially, we would like to be able to explore the concept of \emph{modularity}, and it turns out that this property is meaningfully preserved by the translation to a decision structure. This is the focus of the next section.

\section{Modularity in decision structures} \label{sec:modularity}

In this section we formalise modularity of control architectures. Modularity of a system informally suggests that the system is composed of simple parts, interacting via fixed interfaces. Here we build upon the definition of a module in a graph to construct modules in decision structures, which are in essence subgraphs obeying a fixed interface with the rest of the graph.  This definition allows decision structures to be broken down into simpler pieces, a statement which we quantify using a complexity measure we derive from that of~\cite{mccabe1976complexity}. This concept captures intuitive notions of `modular subparts' in $k$-BTs and DTs, as in these cases modules correspond precisely to subtrees. In addition, we characterise the $k$-BTs, BTs, TRs and DTs by the properties of their modules.

\subsection{Modules}
Let us begin with the following definition of a module in a directed graph, which will inspire the definition of a module in a decision structure.
\begin{defn}[Modules in graphs~\cite{gallai1967transitiv}]
Let $G=(N,A)$ be a graph. A subset $X\subseteq N$ is a \emph{module} if for every $v\not\in X$, $v$ either has arcs to all vertices of $X$ or none, and $v$ either receives arcs from all vertices of $X$ or none.
\end{defn}

A module in a graph can therefore be thought of as a subgraph with a uniform `interface' with the rest of the graph. This interface is as general as possible, with all arcs in or out of this subgraph going to all or none of the nodes of the module. Graph-theoretic modules have been found to be useful for improving the efficiency of a number of graph algorithms~\cite{mcconnell1999modular}. For undirected graphs, there is a canonical decomposition of the graph into nested modules, called the modular decomposition, which is useful for recognising a number of graph classes~\cite{mcconnell1999modular}. This decomposition is constructed by repeatedly taking the \emph{graph quotient} over a partition induced by modules, where the modules become \emph{factors}. These notions are defined next.
\begin{defn}
Let $G$ be a graph, and $P=\{S_1,\dots,S_n\}$ a partition of its node set. The \emph{quotient graph} $G/P$ is the graph whose node set is $P$ and where there is an arc $(S_i, S_j)\in A(G/P)$ if and only if there exists nodes $n_i\in S_i, n_j\in S_j$ with $(n_i,n_j)\in A(G)$. The subgraphs $G[S_1], G[S_2],\dots$ are called the \emph{factors} of $G$.
\end{defn}

Now we extend this idea to decision structures, by finding an appropriate `interface' and thus a definition of a module. Many of the following definitions will be named after their corresponding concept for modules in graph theory, and a number of our results will be similar to corresponding results in the graph theory literature.

\begin{defn}[Modules in decision structures] \label{def:modules}
Let $Z$ be a decision structure. Let $X\subseteq N(Z)$ be a subset where $Z[X]$ is also a decision structure. We say $X$ is a \emph{module} if for every node $v\in N(Z)\setminus X$, any arc from $v$ into $X$ goes to $X$'s source, and if there is an arc labelled $\bm{r}$ out of $X$ to $v$, then for every $x\in X$ the $\bm{r}$ out of $x$ exists and goes either to $v$ or to another element of $X$.
\end{defn}

We often think of modules via the subgraphs they induce in the decision structure, so we can refer to their `sources' and `internal arcs' without confusion. This `interface' can be summarised informally in the following way. Arcs into a module go to its source, and if any arc exits the module with a specific return value $\bm{r}$, it must go to a specific node. In other words, whenever we leave a module, the subsequent node is determined by only the return value on which we left the module. Just as in the semantics of decision structure we cannot select the action labelling a node if the return value matches an out-arc, we cannot select the entire module if, treated as an individual ASM, it would return a value $\bm{r}$ which matches any arc out of that module. Modules are therefore node subsets which are `independent' of the rest of the decision structure. Modules depend on arc labels, but are completely independent of node labels, so any two structurally equivalent decision structures have the same modules.
\begin{remark}[Motivation for modules]
For some motivation behind Definition~\ref{def:modules}, consider a function call in a structured programming
language. We can think of modules in the same way, as a function call from the `outer' decision structure to the module. Al arcs into the module go to the source, just as functions have a single entry point. Moreover, all arcs labelled $\bm{r}$ out of the module go to the same node, so the outer decision structure is independent of which node in the node produced that value. Similarly, a function call always returns to the same point in the calling program, with subsequent execution determined only by the value it returns. In fact, this relationship can be made precise. If we interpret an action in a decision structure as a function call to another decision structure, then the result is equivalent to a \emph{module expansion} (Definition~\ref{def:moduleexpansion}) of the outer structure by the inner. The correctness of this interpretation is Theorem~\ref{contraction}. As further motivation, consider a subclassing (inheritance) relationship in object-oriented programming. If a method is not handled specifically by a subclass, the superclass implementation is used. In the same way, if a module does not handle internally any particular return value, the structure containing it should handle the value, agnostic to which part of the module produced it.
\end{remark}
As shown later, Definition~\ref{def:modules} allows us to construct a decomposition of the decision structure by graph quotient, with the modules as factors, where the original decision structure can be reconstructed from its decomposition. Further, we will show this decomposition is unique.

Modules will provide us with a method to distinguish between more and less modular structures. However, the mere existence of modules does not provide any useful information, because all decision structures have at least some \emph{trivial modules}.
\begin{lemma} \label{trivialmodules}
For any decision structure $Z$, the sets $N(Z)$ and $\{v\}$ for any $v\in N(Z)$ are modules. 
\end{lemma}
\begin{defn}
We call the above modules \emph{trivial modules}, and a structure with only trivial modules is called \emph{prime}.
\end{defn}
 Clearly a prime decision structure is in a sense the least modular structure possible---there is no way of breaking it down into simpler pieces. We will later develop a complexity measure which determines the complexity of a structure by the complexity of its most complex prime module. Using such a measure, a prime structure is as complex as it is---it cannot be decomposed into a simpler structure by breaking it into non-trivial modules. However, before we can define this there is a more immediate question, which is whether we can compute the modules of a structure. It turns out this can be done efficiently. 
 
 \begin{thm} \label{moduletimecomplexity}
Let $Z$ be a decision structure, with $|N(Z)|=n$, with $k$ distinct arc labels. The set of all modules in $Z$ can be found in time $O(n^2k)$.
\end{thm}
\begin{proof}
An algorithm to find all modules in a decision structure is given as Algorithm~\ref{alg:findmodules} in the Appendix. The proof of its correctness and time complexity is also in the Appendix.
\end{proof}
In order to explain how a structure is decomposed into modules, we will need the following definitions.
\begin{defn} 
Let $Z$ be a decision structure. We call a module \emph{maximal} if it is a proper subset of $N(Z)$ and is contained in no other module except $N(Z)$. A \emph{modular partition} $P$ is a partition of $N(Z)$ formed by modules, that is, every element of the partition is a module. A \emph{maximal partition} is a modular partition where all modules are maximal.
\end{defn} 
 \begin{lemma} \label{prime if maximal}
 Let $Z$ be a decision structure and $P$ a modular partition. Then the quotient $Z/P$ is also a decision structure. Moreover, if $P$ is maximal then $Z/P$ is prime.
 \end{lemma}
 To begin, we formalise how a module can be treated as an independent subpart. We show that we can contract modules to individual nodes within a structure, and reverse this transformation by expanding nodes into modules. 
 \begin{defn}[Module contraction]
 Let $Z$ be a decision structure, and $Q$ a module of $Z$. The \emph{module contraction} of $Q$ in $Z$ is the graph $Z/\{Q,\{v_1\},\dots,\{v_m\}\}$ where $\{v_1,\dots,v_m\}=N(Z)\setminus Q$. For simplicity we shall write this as $Z/Q$, as the modular partition $\{Q,\{v_1\},\dots,\{v_m\}\}$ is defined implicitly by $Q$.
 \end{defn}
 The module contraction is the graph formed by treating $Q$ as a single node in $Z$. Because all arcs into and out of a module go to specific nodes depending on their label, this contraction is reversible. In other words, without any additional information, we can recover $Z$ from $Z/Q$ given $Q$ and the node that needs to be `expanded' into $Q$. The following definition formalises this.
 \begin{defn}[Module expansion] \label{def:moduleexpansion}
 Let $Z$ and $Q$ be decision structures, with $v\in N(Z)$. The \emph{module expansion} of $Q$ at $v$, written $Z\cdot^v Q$ is the graph formed as follows. Replace $v$ with a copy of $Q$, where arcs into $v$ in $Z$ go to $Q$'s source in $Z\cdot^v Q$. For every arc \begin{tikzcd} v \ar[r,"\bm{r}"] & z\end{tikzcd} in $Z$, add arcs \begin{tikzcd} q \ar[r,"\bm{r}"] & z\end{tikzcd} from every node $q\in Q$ which does not already have an $\bm{r}$ arc within $Q$.
 \end{defn}
 Essentially, we are adding precisely the requirements for $N(Q)$ to be a module within $Z$, and $(Z\cdot^v Q)[N(Q)] \cong Q$, with no modifications to nodes other than $v$. It is straightforward to see that $N(Q)$ is a module of $Z\cdot^v Q$ and further $(Z\cdot^v Q)/N(Q)\cong Z$. Figure~\ref{fig:trans_exmp} shows an example of a module expansion.
 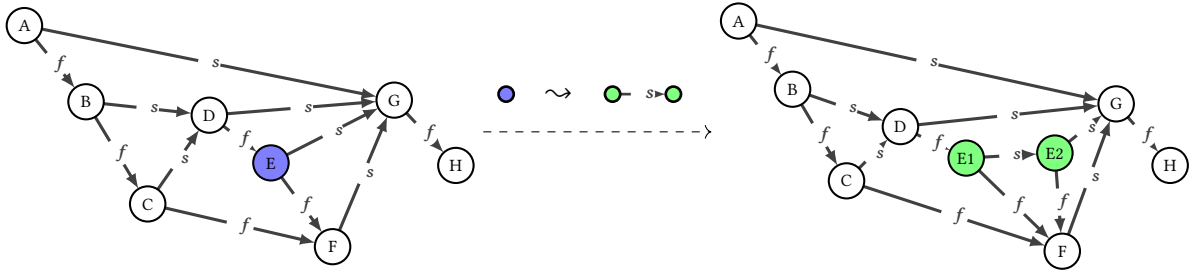
\begin{figure}
    \centering
    \begin{tikzpicture}
\clip (0,0) rectangle (8.0,5.0);
\Vertex[x=7.650,y=2,opacity=0,fontscale=1.2,IdAsLabel]{H}
\Vertex[x=6.607,y=3.114,opacity=0,fontscale=1.2,IdAsLabel]{G}
\Vertex[x=5.564,y=0.614,opacity=0,fontscale=1.2,IdAsLabel]{F}
\Vertex[x=4.521,y=2.043,color=blue,opacity=0.5,fontscale=1.2,IdAsLabel]{E}
\Vertex[x=3.479,y=2.857,opacity=0,fontscale=1.2,IdAsLabel]{D}
\Vertex[x=2.436,y=1.343,opacity=0,fontscale=1.2,IdAsLabel]{C}
\Vertex[x=1.393,y=3.086,opacity=0,fontscale=1.2,IdAsLabel]{B}
\Vertex[x=0.350,y=4.386,opacity=0,fontscale=1.2,IdAsLabel]{A}
\Edge[,label=$\bm{f}$,fontscale=1.2,Direct](G)(H)
\Edge[,label=$\bm{s}$,fontscale=1.2,Direct](F)(G)
\Edge[,label=$\bm{s}$,fontscale=1.2,Direct](E)(G)
\Edge[,label=$\bm{f}$,fontscale=1.2,Direct](E)(F)
\Edge[,label=$\bm{s}$,fontscale=1.2,Direct](D)(G)
\Edge[,label=$\bm{f}$,fontscale=1.2,Direct](D)(E)
\Edge[,label=$\bm{s}$,fontscale=1.2,Direct](C)(D)
\Edge[,label=$\bm{f}$,fontscale=1.2,Direct](C)(F)
\Edge[,label=$\bm{s}$,fontscale=1.2,Direct](B)(D)
\Edge[,label=$\bm{f}$,fontscale=1.2,Direct](B)(C)
\Edge[,label=$\bm{s}$,fontscale=1.2,Direct](A)(G)
\Edge[,label=$\bm{f}$,fontscale=1.2,Direct](A)(B)
\end{tikzpicture}
    \begin{tikzpicture}
    \SetEdgeStyle[LineWidth=1pt]
    \clip (0,0) rectangle (3,5);
    \draw[->,dashed] (0,2) -- (3,2);
    \Vertex[x=0.3,y=2.5,size=0.2,color=blue,opacity=0.5]{Q};
    \Vertex[x=1.7,y=2.5,size=0.2,color=green,opacity=0.5]{V1};
    \Vertex[x=2.5,y=2.5,size=0.2,color=green,opacity=0.5]{V2};
    \Edge[,Direct,label=$\bm{s}$](V1)(V2);
    \node at (1,2.5) {$\leadsto$};
    
    \end{tikzpicture}
    \begin{tikzpicture}
\clip (0,0) rectangle (8.0,5.0);
\Vertex[x=7.650,y=2,opacity=0,fontscale=1.2,IdAsLabel]{H}
\Vertex[x=6.737,y=3.050,opacity=0,fontscale=1.2,IdAsLabel]{G}
\Vertex[x=5.825,y=0.537,opacity=0,fontscale=1.2,IdAsLabel]{F}
\Vertex[x=5.7,y=2.212,opacity=0.5,color=green,fontscale=1.2,IdAsLabel]{E2}
\Vertex[x=4.2,y=2.125,opacity=0.5,color=green,fontscale=1.2,IdAsLabel]{E1}
\Vertex[x=3.087,y=2.662,opacity=0,fontscale=1.2,IdAsLabel]{D}
\Vertex[x=2.175,y=1.737,opacity=0,fontscale=1.2,IdAsLabel]{C}
\Vertex[x=1.262,y=3.300,opacity=0,fontscale=1.2,IdAsLabel]{B}
\Vertex[x=0.350,y=4.463,opacity=0,fontscale=1.2,IdAsLabel]{A}
\Edge[,label=$\bm{f}$,fontscale=1.2,Direct](G)(H)
\Edge[,label=$\bm{s}$,fontscale=1.2,Direct](F)(G)
\Edge[,label=$\bm{s}$,fontscale=1.2,Direct](E2)(G)
\Edge[,label=$\bm{f}$,fontscale=1.2,Direct](E2)(F)
\Edge[,label=$\bm{s}$,fontscale=1.2,Direct](E1)(E2)
\Edge[,label=$\bm{f}$,fontscale=1.2,Direct](E1)(F)
\Edge[,label=$\bm{s}$,fontscale=1.2,Direct](D)(G)
\Edge[,label=$\bm{f}$,fontscale=1.2,Direct](D)(E1)
\Edge[,label=$\bm{s}$,fontscale=1.2,Direct](C)(D)
\Edge[,label=$\bm{f}$,fontscale=1.2,Direct](C)(F)
\Edge[,label=$\bm{s}$,fontscale=1.2,Direct](B)(D)
\Edge[,label=$\bm{f}$,fontscale=1.2,Direct](B)(C)
\Edge[,label=$\bm{s}$,fontscale=1.2,Direct](A)(G)
\Edge[,label=$\bm{f}$,fontscale=1.2,Direct](A)(B)
\end{tikzpicture}
    \caption{An example of a module expansion, where the decision structure \arcS{} is expanded at the node E. Both decision structures correspond to BTs, and this expansion can be thought of as replacing an action by a subtree of two actions rooted by a Sequence operator.}
    \label{fig:trans_exmp}
\end{figure}

\subsection{The module decomposition}
 In this section we study the decomposition of decision structures into modules, which is the process of breaking down decision structures into their constituent modules. The process involves repeatedly constructing modular partitions then taking the graph quotient with these modules as factors, then repeating the process on the factors until all factors are prime. In graph theory, undirected graphs have been shown to possess a unique decomposition into modules. We show here that decision structures do likewise.
 
\begin{lemma} \label{uniquedecomposition}
For any decision structure $Z$, exactly one of the following is true.
\begin{enumerate}
    \item there exists a unique maximal partition, or
    \item there exists a unique modular partition $P$ such that the quotient $Z/P$ is isomorphic to a path of length at least two with all arcs having the same label, and that path has maximal length among all modular partitions $R$ for which $Z/R$ is a path.
\end{enumerate}
\end{lemma}

Maximal partitions generally form a canonical choice for the decomposition, as their quotients are prime. This lemma shows that in the one case where this choice is not unique, there is instead a unique partition which constructs a maximum length path as its quotient.

\begin{defn}
Let $Z$ be a decision structure. If case (1) of above holds for $Z$, let $P$ be the maximal partition. If instead case (2) holds, let $P$ be the unique modular partition which gives the longest path. Construct the quotient graph $Z/P$. Repeat this recursively on each factor, until every factor is a trivial module. We call this the \emph{module decomposition} of $Z$.
\end{defn}

This provides a nested set of quotient graphs, as shown in Figures~\ref{fig:motivation} and~\ref{fig:bt_and_decision}. We obtain the following result.
\begin{lemma} \label{primequotients}
Every quotient graph in the module decomposition is either prime or is a path of length at least two with the same label on all arcs.
\end{lemma}

The uniqueness of the module decomposition follows immediately from Lemma~\ref{uniquedecomposition}. The choice of partition in case (2) of Lemma~\ref{uniquedecomposition} is a canonical one, as it corresponds directly in the $k$-BT case to the unique $k$-BT in \emph{compressed form}, as discussed earlier. That is, where each internal node has at least two children and all children have different root operators than their parent. For BTs these are trees where the Success and Failure nodes alternate at every tier in the tree. The next result justifies constructing the decomposition recursively.

\begin{lemma} \label{moduleinmodule}
Let $Z$ be a decision structure, $Y$ a module in $Z$, and $X\subseteq Y$. Then $X$ is a module in $Z$ if and only if $X$ is a module in $Z[Y]$.
\end{lemma}

This allows us to conclude that any module found anywhere in the module decomposition is a module of the entire graph. In fact, except for the case of paths of length greater than two, for which \emph{every connected subset is a module}, the converse also holds; every subgraph induced by any module is a quotient graph in the module decomposition. As an example, consider the decompositions in Figure \ref{fig:bt_and_decision}. The non-trivial modules are $\{a,b\},\{h,i\},\{d,e\},\{e,f\},\{d,e,f\},\{d,e,f,g\},\{d,e,f,g,h,i\}$ and $\{c,d,e,f,g,h,i\}$. Their subgraphs are all quotients in the module decomposition, except for $\{d,e\}$ and $\{e,f\}$, which can be derived from the length-two path $\{d,e,f\}$.

Observe that any modular partition can be formed by a sequence of module contractions. In fact, the entire module decomposition can be formed by a sequence of module contractions, from the smaller to the larger modules. Similarly, we can construct the structure by a sequence of module expansions from the one-node graph, where the graphs expanded are precisely the quotients of the decomposition from largest to smallest.

\subsection{Characterising modular architectures} \label{sec:characterisations}

In Section \ref{sec:decisionstructures} we showed how BTs, DTs, TRs and $k$-BTs could be translated into decision structures. However, it was not necessarily clear whether a given decision structure was structurally equivalent to an architecture in one of these classes. In this section we show how the definition of a module allows for an elegant characterisation of the decision structures which are structurally equivalent to these architectures.

\begin{thm} \label{kbt_characterisation}
Let $Z$ be a decision structure with $k$ distinct arc labels. Then $Z$ is structurally equivalent to a $k$-BT if and only if every quotient graph in $Z$'s module decomposition is a path.
\end{thm}

In fact, this theorem shows that the modules in the module decomposition correspond precisely to subtrees of the $k$-BT's unique compressed form representation. This is shown in Figure~\ref{fig:bt_and_decision}. This result further motivates our definition of a module, as it represents precisely the `well-defined subparts' in the $k$-BT case. Due to the relationship between $k$-BTs, BTs and TRs, we obtain the following corollaries immediately.
\begin{corollary}
A decision structure $Z$ is structurally equivalent to a BT or TR respectively if and only if every module in its decomposition is a path and it is labelled by no more than 2 or 1 distinct labels, respectively.
\end{corollary}

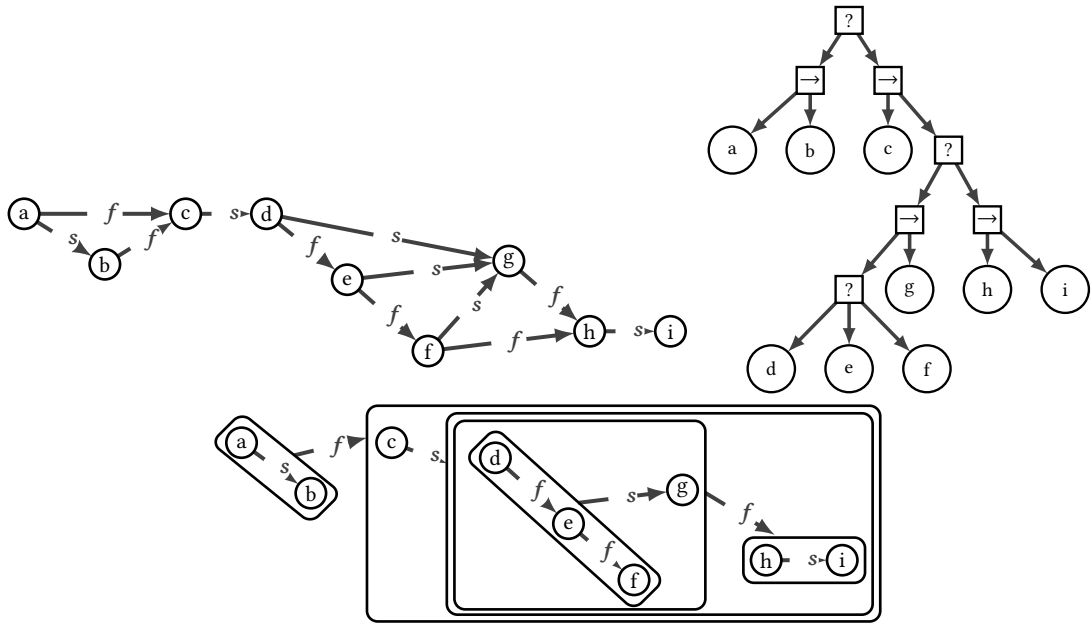
\begin{figure}
    \centering
    \begin{tikzpicture}
\clip (0,0) rectangle (9.0,3.0);
\Vertex[x=8.750,y=0.852,size=0.4,opacity=0,fontscale=1.2,IdAsLabel]{i}
\Vertex[x=7.688,y=0.852,size=0.4,opacity=0,fontscale=1.2,IdAsLabel]{h}
\Vertex[x=6.625,y=1.784,size=0.4,opacity=0,fontscale=1.2,IdAsLabel]{g}
\Vertex[x=5.562,y=0.590,size=0.4,opacity=0,fontscale=1.2,IdAsLabel]{f}
\Vertex[x=4.500,y=1.536,size=0.4,opacity=0,fontscale=1.2,IdAsLabel]{e}
\Vertex[x=3.438,y=2.410,size=0.4,opacity=0,fontscale=1.2,IdAsLabel]{d}
\Vertex[x=2.375,y=2.410,size=0.4,opacity=0,fontscale=1.2,IdAsLabel]{c}
\Vertex[x=1.312,y=1.740,size=0.4,opacity=0,fontscale=1.2,IdAsLabel]{b}
\Vertex[x=0.250,y=2.410,size=0.4,opacity=0,fontscale=1.2,IdAsLabel]{a}
\Edge[,label=$\bm{s}$,fontscale=1.2,Direct](h)(i)
\Edge[,label=$\bm{f}$,fontscale=1.2,Direct](g)(h)
\Edge[,label=$\bm{s}$,fontscale=1.2,Direct](f)(g)
\Edge[,label=$\bm{f}$,fontscale=1.2,Direct](f)(h)
\Edge[,label=$\bm{s}$,fontscale=1.2,Direct](e)(g)
\Edge[,label=$\bm{f}$,fontscale=1.2,Direct](e)(f)
\Edge[,label=$\bm{s}$,fontscale=1.2,Direct](d)(g)
\Edge[,label=$\bm{f}$,fontscale=1.2,Direct](d)(e)
\Edge[,label=$\bm{s}$,fontscale=1.2,Direct](c)(d)
\Edge[,label=$\bm{f}$,fontscale=1.2,Direct](b)(c)
\Edge[,label=$\bm{s}$,fontscale=1.2,Direct](a)(b)
\Edge[,label=$\bm{f}$,fontscale=1.2,Direct](a)(c)
\end{tikzpicture}
    \begin{tikzpicture}
\clip (0,0) rectangle (6.0,6.0);
\Vertex[x=2.264,y=5.600,size=0.4,opacity=0,label=?,fontscale=1.2,shape=rectangle]{1}
\Vertex[x=1.688,y=4.704,size=0.4,opacity=0,label=$\rightarrow$,fontscale=1.2,shape=rectangle]{2}
\Vertex[x=0.536,y=3.664,size=0.7,opacity=0,label=a,fontscale=1.2,shape=ellipse]{3}
\Vertex[x=1.688,y=3.664,size=0.7,opacity=0,label=b,fontscale=1.2,shape=ellipse]{4}
\Vertex[x=2.840,y=4.704,size=0.4,opacity=0,label=$\rightarrow$,fontscale=1.2,shape=rectangle]{5}
\Vertex[x=2.840,y=3.664,size=0.7,opacity=0,label=c,fontscale=1.2,shape=ellipse]{6}
\Vertex[x=3.736,y=3.664,size=0.4,opacity=0,label=?,fontscale=1.2,shape=rectangle]{7}
\Vertex[x=3.160,y=2.624,size=0.4,opacity=0,label=$\rightarrow$,fontscale=1.2,shape=rectangle]{8}
\Vertex[x=2.264,y=1.584,size=0.4,opacity=0,label=?,fontscale=1.2,shape=rectangle]{9}
\Vertex[x=1.112,y=0.400,size=0.7,opacity=0,label=d,fontscale=1.2,shape=ellipse]{10}
\Vertex[x=2.264,y=0.400,size=0.7,opacity=0,label=e,fontscale=1.2,shape=ellipse]{11}
\Vertex[x=3.416,y=0.400,size=0.7,opacity=0,label=f,fontscale=1.2,shape=ellipse]{12}
\Vertex[x=3.160,y=1.584,size=0.7,opacity=0,label=g,fontscale=1.2,shape=ellipse]{13}
\Vertex[x=4.312,y=2.624,size=0.4,opacity=0,label=$\rightarrow$,fontscale=1.2,shape=rectangle]{14}
\Vertex[x=4.312,y=1.584,size=0.7,opacity=0,label=h,fontscale=1.2,shape=ellipse]{15}
\Vertex[x=5.464,y=1.584,size=0.7,opacity=0,label=i,fontscale=1.2,shape=ellipse]{16}
\Edge[,fontscale=1.2,Direct](1)(2)
\Edge[,fontscale=1.2,Direct](1)(5)
\Edge[,fontscale=1.2,Direct](2)(3)
\Edge[,fontscale=1.2,Direct](2)(4)
\Edge[,fontscale=1.2,Direct](5)(6)
\Edge[,fontscale=1.2,Direct](5)(7)
\Edge[,fontscale=1.2,Direct](7)(8)
\Edge[,fontscale=1.2,Direct](7)(14)
\Edge[,fontscale=1.2,Direct](8)(9)
\Edge[,fontscale=1.2,Direct](8)(13)
\Edge[,fontscale=1.2,Direct](9)(10)
\Edge[,fontscale=1.2,Direct](9)(11)
\Edge[,fontscale=1.2,Direct](9)(12)
\Edge[,fontscale=1.2,Direct](14)(15)
\Edge[,fontscale=1.2,Direct](14)(16)
\end{tikzpicture}
    \begin{tikzpicture}
\clip (0,0) rectangle (9.0,3.0);
\Vertex[x=8.3,y=0.852,size=0.4,opacity=0,fontscale=1.2,IdAsLabel]{i}
\Vertex[x=7.3,y=0.852,size=0.4,opacity=0,fontscale=1.2,IdAsLabel]{h}
\Vertex[x=6.2,y=1.784,size=0.4,opacity=0,fontscale=1.2,IdAsLabel]{g}
\Vertex[x=5.562,y=0.590,size=0.4,opacity=0,fontscale=1.2,IdAsLabel]{f}
\Vertex[x=4.7,y=1.336,size=0.4,opacity=0,fontscale=1.2,IdAsLabel]{e}
\Vertex[x=3.738,y=2.210,size=0.4,opacity=0,fontscale=1.2,IdAsLabel]{d}
\Vertex[x=2.375,y=2.410,size=0.4,opacity=0,fontscale=1.2,IdAsLabel]{c}
\Vertex[x=1.312,y=1.740,size=0.4,opacity=0,fontscale=1.2,IdAsLabel]{b}
\Vertex[x=0.4,y=2.410,size=0.4,opacity=0,fontscale=1.2,IdAsLabel]{a}
\draw[rounded corners, line width=1pt] (7,1.152) -- (8.6,1.152) -- (8.6,0.552) -- (7,0.552) -- cycle;
\Vertex[Pseudo,size=0,x=7.5,y=1.135]{HI}
\Vertex[Pseudo,size=0,x=2.15,y=2.5]{CDEFGHI}
\draw[rounded corners, line width=1pt] (3.2,2.7) -- (3.2,0.2) -- (6.5,0.2) -- (6.5,2.7) -- cycle;
\draw[rounded corners, line width=1pt] (3.1,2.8) -- (3.1,0.13) -- (8.7,0.13) -- (8.7,2.8) -- cycle;
\draw[rounded corners, line width=1pt] (2.05,2.9) -- (2.05,0.04) -- (8.8,0.04) -- (8.8,2.9) -- cycle;
\Vertex[Pseudo,size=0,x=3.2,y=2.1]{DEFGHI}
\Vertex[Pseudo,size=0,x=6.41,y=1.8]{DEFG}
\draw[rounded corners, line width=1pt] (3.738,2.610) -- (3.338,2.210) -- (5.562,0.190) -- (5.962,0.590) -- cycle;
\Vertex[Pseudo,size=0,x=4.73,y=1.6]{DEF}
\draw[rounded corners, line width=1pt] (0,2.410) -- (1.312,1.340) -- (1.712,1.740) -- (0.4,2.810) -- cycle;
\Vertex[Pseudo,size=0,x=1,y=2.22]{AB}
\Edge[,label=$\bm{s}$,fontscale=1.2,Direct](h)(i)
\Edge[,label=$\bm{f}$,fontscale=1.2,Direct](DEFG)(HI)
\Edge[,label=$\bm{s}$,fontscale=1.2,Direct](DEF)(g)
\Edge[,label=$\bm{f}$,fontscale=1.2,Direct](e)(f)
\Edge[,label=$\bm{f}$,fontscale=1.2,Direct](d)(e)
\Edge[,label=$\bm{s}$,fontscale=1.2,Direct](c)(DEFGHI)
\Edge[,label=$\bm{f}$,fontscale=1.2,Direct](AB)(CDEFGHI)
\Edge[,label=$\bm{s}$,fontscale=1.2,Direct](a)(b)
\end{tikzpicture}
    \caption{A decision structure, its module decomposition and a structurally equivalent BT. Observe the connection between modules in the decision structure and subtrees in the BT.}
    \label{fig:bt_and_decision}
\end{figure}

A similar result is provable for DTs.
\begin{thm} \label{dt_characterisation}
Let $Z$ be a decision structure with two distinct arc labels. Then $Z$ is structurally equivalent to a DT if and only if every quotient graph is isomorphic to the structure with a single source and two sinks, and the factor corresponding to a source of any quotient graph consists of a trivial module containing one node.
\end{thm}
In fact, it is straightforward to show that the modules in a decision tree (or any tree-structured decision structure) are precisely the subtrees, regardless of edge labels.

These results allow us to identify which decision structures correspond to these particular architectures, and so allows one to prove interesting properties of the architectures. In addition, they connect the modular subparts of $k$-BTs and DTs with modules in their decision structures and so allow the results of Section~\ref{sec:verrification} to apply immediately to them. We will discuss some of the consequences of these below.

When given a decision structure, the aforementioned methods identify whether it has a structurally equivalent interpretation as a BT, DT, or the other modular control architectures, in a computationally tractable manner.
\begin{thm} \label{testingtimecomplexity}
Testing whether an individual structure $Z$ is a BT, TR, DT or $k$-BT for some $k$ can be done in time $O(n^2)$.
\end{thm}
Essentially, for $k$-BTs, Algorithm \ref{alg:findmodules} is used to find all modules and check that they are paths. For DTs this is easier, as they are precisely binary trees with two labels and so can be identified by a graph traversal. The full proof is in the Appendix. To illustrate the applicability of this Theorem, consider the decision structure in Figure~\ref{fig:bt_and_decision}. This decision structure is equivalent to a BT, though this would be difficult to discern without such a BT being provided. With the tools discussed so far, we can identify the equivalent tree, also shown in Figure~\ref{fig:bt_and_decision}, in quadratic time. Consider also the complementary problem of identifying decision structures which \emph{do not} correspond to BTs. In Figure~\ref{fig:not_bt} a simple decision structure $T$ is shown. This does not correspond to any BT, but it would be difficult to prove this without these results. With this result, we can confirm this by showing that the only non-trivial module of $T$ is $\{a,b,c,d\}$, whose induced subgraph is not a path.
 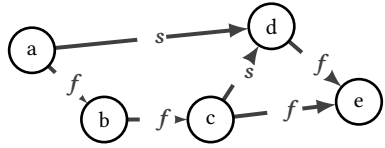
\begin{figure}
    \centering
    \begin{tikzpicture}
\clip (0,0) rectangle (5.0,2.0);
\Vertex[x=4.650,y=0.617,opacity=0,fontscale=1.2,IdAsLabel]{e}
\Vertex[x=3.5,y=1.618,opacity=0,fontscale=1.2,IdAsLabel]{d}
\Vertex[x=2.7,y=0.382,opacity=0,fontscale=1.2,IdAsLabel]{c}
\Vertex[x=0.350,y=1.3,opacity=0,fontscale=1.2,IdAsLabel]{a}
\Vertex[x=1.3,y=0.382,opacity=0,fontscale=1.2,IdAsLabel]{b}
\Edge[,label=$\bm{f}$,fontscale=1.2,Direct](d)(e)
\Edge[,label=$\bm{s}$,fontscale=1.2,Direct](c)(d)
\Edge[,label=$\bm{f}$,fontscale=1.2,Direct](c)(e)
\Edge[,label=$\bm{s}$,fontscale=1.2,Direct](a)(d)
\Edge[,label=$\bm{f}$,fontscale=1.2,Direct](a)(b)
\Edge[,label=$\bm{f}$,fontscale=1.2,Direct](b)(c)
\end{tikzpicture}
    \caption{A decision structure $T$ that is not equivalent to any BT.}
    \label{fig:not_bt}
\end{figure}
This example shows that the set of decision structures corresponding to BTs is difficult to characterise without this theorem. One could derive some necessary conditions for decision structures to be BTs, such as that they must have a directed path through all nodes and all arcs labelled by $\bm{s}$ and $\bm{f}$. In addition, one can prove that all such decision structures are \emph{upward planar}. While necessary, the above conditions are not sufficient, as Figure~\ref{fig:not_bt} demonstrates. One might wonder to what degree the BT-ness of a decision structure is a property of the return values on its arcs. Certainly, these values matter, but the graph of Figure \ref{fig:not_bt} is not equivalent to a BT for \emph{any possible arc-labelling}, showing that the property of being equivalent to a BT does not only depend on the arc labels but also on the unlabelled graph structure.
\begin{remark}[Prior characterisations of Behavior Trees] \label{rem:hannaford}
The above result contradicts the result in~\cite{hannaford2019hidden}, which presents a characterisation of graphs corresponding to BTs. The construction of these graphs from BTs was essentially the same as our construction map for a BT, with the addition of two nodes representing the entire tree returning Success or Failure. We can translate to and from such graphs to our decision structures by adding two additional nodes and adding an arc from each node without an $\bm{f}$ arc to the Failure node and each node without an $\bm{s}$ arc to the Success node. Their result incorrectly claims that the class of graphs to which BTs are equivalent under this construction is the whole class of single-source digraphs with two sinks and a directed path through all non-sink nodes, which each have degree two. Figure \ref{fig:not_bt} provides a  counterexample, when appropriately translated, showing that these conditions are only necessary and not sufficient. Note that if the additional `Success' and `Failure' nodes are added, we can prove a nice result, which is that each decision structure either has no labellings of $\bm{s}$ and $\bm{f}$ to its arcs corresponding to BTs, or it has precisely two. These correspond respectively to a specific BT and its negation (in the sense of~\cite{biggar2020framework}). A variant of this is proved in~\cite{hannaford2019hidden}.
\end{remark}
\begin{remark}[Decision structures and Finite State Machines] \label{rem:fsms}
Decision structures are in a sense reactive Finite State Machines. To be more precise, if a transition is added from every node in the decision structure to the source node, and these transitions are taken after each prescribed update step, we obtain a FSM (specifically a \emph{clocked sequential system}~\cite{cavanagh2018sequential}) which has the same ASM as the original structure. This is essentially the translation given in~\cite{generalise,btbook}, though there it is defined using Hierarchical FSMs. We now determine which such FSMs are BTs under this translation: they are exactly those which, when arcs into the start node are removed, form a decision structure in this class. This is to our knowledge the first result in identifying a subclass of FSMs corresponding to BTs. We also note that this equally applies to any of the architectures we have characterised here. This gives intuition towards which FSMs are `modular', though we should note that this translation is not a structural equivalence, because these additional arcs to the source do not correspond to return values of the input actions, and in general the execution models of FSMs and decision structures are slightly different. However, analogous results on modularity in FSMs may be possible in future work, using a definition of a module more tailored to FSMs. See Section~\ref{sec:conclusions}.
\end{remark}
\subsection{A complexity measure for decision structures} \label{sec:complexitymeasure}

In this section we develop a measure of decision structure complexity which incorporates modularity. Our approach follows that of McCabe~\cite{mccabe1976complexity}, where a complexity measure called \emph{cyclomatic complexity} (based on the \emph{cyclomatic number} in graph theory~\cite{berge2001theory}) for control-flow graphs is defined. This counts the number of linearly independent (undirected) cycles in the graph plus one, a number which essentially measures the degree of branching of the program. McCabe then defines a measure called \emph{essential complexity} which is the cyclomatic complexity of the structure after it has been `reduced' by repeatedly contracting subgraphs with single entry and exit nodes. Cyclomatic complexity has been shown to be correlated with the difficulty of testing a piece of code, and has become a standardised metric of code complexity~\cite{watson1996structured}.

There is a clear connection between this idea and the module decomposition of a decision structure. Hence, given the applicability that this measure has found in software engineering, we shall define a similar concept for decision structures. Indeed, because decision structures are graphs, the definition of cyclomatic complexity for decision structures is essentially identical to that of control-flow graphs.

\begin{defn}
Let $Z$ be a decision structure. The \emph{cyclomatic complexity} of $Z$ is the number of linearly independent undirected cycles in $Z$, plus one.
\end{defn}
\begin{lemma} \label{calculatincyclcomplexity}
Let $Z$ be a decision structure with $s$ sinks, and $|N(Z)| = n$ and $|A(Z)| = a$. Then the cyclomatic complexity of $Z$ is $a+s-n+1$.
\end{lemma}
\begin{remark}
This definition is chosen for consistency with McCabe's definition for control-flow graphs, which are digraphs with a single source and sink. Decision structures are `control-flow graphs' in McCabe's sense if they have one sink; in general they can have many sinks, but we deal with the issue in this case by adding a single additional node to be the single sink and add arcs from all existing sinks to this node. Our formula then is equal to the cyclomatic complexity (in McCabe's sense) of this resultant control-flow graph. As for control-flow graphs, it is easy to prove that a decision structure is a directed path if and only if it has cyclomatic complexity 1.
\end{remark}
This has yet to incorporate any reference to modules. In the same way that subgraphs with single entry and exit nodes can be considered function calls in structured programming, we consider a module in a decision structure to be a function call to another decision structure. Hence, we seek a measure which incorporates the complexity of the most complex part of the structure.
\begin{defn}
Let $Z$ be a decision structure. The \emph{essential complexity} of $Z$ is the maximum cyclomatic complexity of any quotient graph in its module decomposition.
\end{defn}
Note that if we handled case (2) of Lemma \ref{uniquedecomposition} by choosing any possible maximal partition, we would end up with a tree with the same module complexity, as a path has complexity 1. This definition fits well with intuition about how modularity should work, and further it can be easily computed.
\begin{lemma} \label{essentialtimecomplexity}
For any decision structure $Z$ with $|N(Z)|=n$ and $k$ distinct labels, its essential complexity can be computed in time $O(n^2k)$.
\end{lemma}

Immediately we can determine the essential complexity of structures which we have so far encountered in this paper. Figure \ref{fig:motivation} has cyclomatic complexity 10 and essential complexity 2. The decision structure in Figure \ref{fig:bt_ss} has cyclomatic complexity 6 and essential complexity 1. Figure \ref{fig:bt_and_decision} has cyclomatic complexity 5 and essential complexity 1. Note immediately that the essential complexity is always no greater than the cyclomatic complexity. The structures in these examples break down into many simple modules, so their essential complexity is small. This makes conceptual sense. Figures \ref{fig:bt_ss} and \ref{fig:bt_and_decision} are structurally equivalent to BTs, and so their essential complexity must be 1. In fact, we can prove another characterisation of the $k$-BTs using this measure.
\begin{thm} \label{essential_characterisation}
Let $Z$ be a decision structure with $k$ distinct arc labels. $Z$ is equivalent to a $k$-BT if and only if it has essential complexity 1.
\end{thm}
This states that $k$-BTs are precisely the decision structures with minimal complexity. We believe this provides an argument for interpreting minimally-complex decision structures as $k$-BTs by default.

\section{Modular formal verification of decision structures} \label{sec:verrification}

The goal of formal verification is to discover system errors or to provide trust in the form of a certificate that a system operates correctly. In other words, given a model of a system and a description of the desired specification of that system, we wish to produce either a counterexample of a legal system execution violating the specification or a proof that no such executions are possible. One challenging aspect of verification is the existence of a trade-off between computational complexity and model fidelity. This is part of a more general problem: how does one build large-scale correct and complex systems, while complexity scales with size? The engineering solution to this problem is abstraction and modularity. Specifically one ensures that each subsystems interacts with only a few others, and modifications to one such system have limited flow-on effects. Applying these ideas to formal verification is a major goal of this paper, and a motivation behind our investigation into modules in decision structures.

We seek modularity in verification in the following way. Suppose a system is guaranteed to be correct with regard to some specification. Then, a modification is made to some component of this system---in this case, a modification of the decision structure. If this change is restricted to a well-defined subpart---that is, a module---we seek criteria under which we can deduce the correctness of the altered structure by considering only this local change. We show that replacing a module in a decision structure with a different module causes no more of an effect than replacing an action by another action. Importantly, this result is independent of any verification scheme or any specific representation of the specification or the system. It can be applied to any and all techniques which are able to verify a decision structure, including special cases such as BTs, $k$-BTs, DTs and TRs. The special case of this theorem for the simple verification scheme discussed later was proved for BTs in~\cite{biggar2020framework}. We now show the result in generality.

 Given a decision structure $Z$, we write $Z(w)$ for the action selected by $Z$ in $w$, and $(Z_B,Z_R)$ for the derived action of $Z$. As modules can be considered decision structures and therefore ASMs in their own right, we will use the same notation for modules. The following theorem proves that the operation of contracting modules preserves the decision structure as an ASM.
\begin{thm} \label{contraction}
Let $Z$ be a decision structure, with $H$ a module of $Z$. Let $Z/H$ be the module contraction of $H$, where $H$ is labelled by its derived action $(H_B,H_R)$ in $Z/H$ and the actions labelling other nodes are unchanged. Then $Z_B=(Z/H)_B$.
\end{thm}
Though seemingly understated, this theorem will allow us to treat modules as individual nodes in \emph{any} verification scheme. We will now specify what we mean by a \emph{verification scheme}. A verification scheme is a function which takes a decision structure (or a model of such) and a formal specification in some form and returns either True or False depending on whether that system satisfies the specification. We will assume only that a verification scheme is deterministic, in that if it is applied to two identical structures and specifications then the output must be the same.
\begin{defn}
Let $V$ be some verification scheme. Let $C_a$ be a predicate, dependent on a decision structure $Z$, a node $v\in N(Z)$, the action $\alpha$ labelling $v$, and another action $\beta$. Let $C_m$ be a predicate dependent on $Z$, a module $H$ in $Z$, and another decision structure $Q$. Let $Z_a$ be the structure given by replacing the label of $v$ by $\beta$, and likewise let $Z_m$ be the structure given by replacing $H$ by $Q$. Formally, $Z_m=(Z/H)\cdot^H Q$. We call $C_a$ a \emph{sufficient condition for actions} if, supposing $Z$ satisfies some specification $\varphi$ according to $V$, whenever $C_a(Z,v,\alpha,\beta)$ is true then $Z_a$ also satisfies $\varphi$ according to $V$. We call $C_m$ a \emph{sufficient condition for modules} if, supposing $Z$ satisfies some specification $\varphi$ according to $V$, whenever $C_m(Z,H,Q)$ is true then $Z_a$ also satisfies $\varphi$ according to $V$.
\end{defn}
One trivial sufficient condition is that $\alpha=\beta$ (for actions) or $H=Q$ (for modules). In both cases, $Z_a=Z_m=Z$, and so the correctness is preserved by the determinism of the verification scheme. In general for a specific verification scheme, there are infinitely many other sufficient conditions.
\begin{thm} \label{verificationcorrespondence}
Let $V$ be a fixed verification scheme. There is a one-to-one correspondence between sufficient conditions for modules and actions.
\end{thm}

Sufficient conditions for actions are quite easy to construct. This is because modifications to a single action are in a sense the simplest possible modification one could make to a decision structure. This theorem shows that modifying modules is just as easy. We will give examples of finding and using sufficient conditions to verify and modify decision structures in the next few sections. From a practical perspective, this is valuable because `modular' architectures are useful precisely because modification does not break existing behavior. This theorem formalises this engineering principle; these structures can be modified in certain ways without breaking formal certificates of correctness. Through the module decomposition, we can identify when this can be done.

A particular set of interesting sufficient conditions is those that are \emph{local}. That is, the condition should depend on the actions and their return values, not how they are situated with the decision structure. Such conditions allow us to improve the efficiency of verification by only requiring computation on the actions, rather than the entire structure. It also allows them to potentially be stored in libraries of verified behavior, as knowing the condition is satisfied between actions allows those actions to be substituted in any decision structure where they are used. We give examples of local conditions in the next section.

\subsection{A verification scheme for reactive ASMs}

To demonstrate the use of the previous result to a practical problem, we will select a specific verification scheme, and develop some sufficient conditions. Our scheme is based off the method for BT verification developed in~\cite{biggar2020framework}, and provides a method for verifying any reactive ASM, including decision structures. It is simple and easy to implement, which makes it useful for explaining these ideas with minimal data. Much of this simplicity comes from its use of abstraction for actions. In this framework, LTL formulas are used as an abstract representation of the behavior of an action interacting with the world. In formal verification it is often assumed that the information of the relations between states induced by the actions is known precisely, to allow for the construction of a transition system representation of the world. Here (following~\cite{biggar2020framework}) we make the weaker assumption that we know this only \emph{abstractly}, consistent with a hierarchical and modular interpretation of these architectures. In other words, our verification scheme depends explicitly on a model of each action, rather than its specific implementation, which may not be known.
\begin{defn}
Let $\alpha$ be an action in $\Act$, $\phi$ an LTL formula, and $\mathcal{W}$ a world. We write $w^\alpha$ for the set of possible subsequent sequences of states after $w$ generated by $\mathcal{W}$ on input $\alpha_B(w)\in\signalspace$. We say $\phi$ \emph{models} $\alpha$ if, for all sequences of states $w_1,w_2,w_3,\dots\in\world$, if $w_2,w_3,\dots \in w_1^\alpha$ then $\phi$ holds in $w_1,w_2,w_3,\dots$. We say $\phi$ is \emph{equivalent to} $\alpha$ if the converse also holds.
\end{defn}
Conceptually, this says that whenever $\alpha$ is selected in a state $w$ to produce a signal into $\mathcal{W}$, $\phi$ must hold over the sequence of states from this point. The formula $\phi$ is an answer to the question: ``what can we guarantee about the state of the world during and after we execute this action?". If we visualise the structure of the world $\mathcal{W}$ as a directed graph, an LTL formula is a set of infinite paths. The statement `$\phi$ models $\alpha$' means that the set of paths defined by $\phi$ is a superset of the paths which can occur following a signal from $\alpha$.

Now we show that modelling individual actions by LTL formulas gives a method to correctly verify \emph{any} ASM, generalising Theorem 5.1 of~\cite{biggar2020framework}. To do this we first recall that an ASM is defined implicitly as a Boolean algebra by a finite set of return values. This means that given some $M: \world\to \Delta$ with $\Delta = \{\alpha_1,\dots,\alpha_n\}$, $M^{-1}(\alpha_i)$ is a subset of $\world$ which is defined by a finite Boolean formula whose variables are of the form ${\alpha_j}_R(w) = \bm{r}$ for $j\in\{1,\dots,n\}$ and $\bm{r}$ in some finite subset $X$ of $\returnvals$. Hence we shall from here interpret this as a formula in the finite Boolean algebra generated by these variables.
\begin{lemma} \label{psi canonical}
Let $M: \world\to \Delta$ be an ASM, with $\Delta = \{\alpha_1,\dots,\alpha_n\}$. Suppose we have LTL formulas $\phi_1,\dots,\phi_n$ equivalent to each action $\alpha_i \in \Delta$. Then $\Psi_M := (M^{-1}(\alpha_1) \land \phi_1)\lor \dots\lor (M^{-1}(\alpha_n)\land \phi_n)$ is equivalent to $(M_B,M_R)$.
\end{lemma}
Though for Theorem \ref{verification works} we will only need the weaker `models' form of the above lemma, we prove a stronger form because it allows us to derive a useful fact: $\Psi_M$ is the strongest possible modelling formula of $M$. If we keep the models $\phi_i$ of each action $\alpha_i$ the same, then any other formula modelling $M$ is entailed by $\Psi_M$. Because $\Psi_M$ is the \emph{canonical} choice in this sense, whenever we have an ASM $M$ selecting from a set $\Delta=\{\alpha_1,\dots,\alpha_n\}$ which have corresponding modelling formulas $\phi_1,\dots,\phi_n$, then we shall assume that $M$ is modelled by $\Psi_M$.

\begin{thm}[Verification of ASMs] \label{verification works}
Let $\alpha$ be an action, and suppose $\phi_\alpha$ models $\alpha$. Let $\varphi$ be an LTL specification. If $\alpha$ acts on $\mathcal{W}$ and $\always \phi_\alpha \models \varphi$, then every sequence of states generated by $\mathcal{W}$ satisfies $\varphi$.
\end{thm}

As ASMs have a unique derived action, the action $\alpha$ mentioned in this theorem could be an ASM made up of many actions. This theorem very closely follows Theorem 5.1 of~\cite{biggar2020framework}, generalised to all reactive ASMs.

\subsection{A sufficient condition in this verification scheme}

Now that we have a concrete verification scheme, we will demonstrate a sufficient condition for modifying actions, which we now know we can apply immediately to modifications of modules. This is by no means the only such, but provides a useful example and will be used extensively in Section \ref{sec:example}.
\begin{lemma} \label{scheme suff condition}
Let $Z$ be a decision structure, with $\alpha$ labelling a node $v$, with $\bm{r}_1,\dots,\bm{r}_m$ labelling the arcs out of $v$. Let $\beta$ be an action, and suppose $\phi_\alpha$ and $\phi_\beta$ model $\alpha$ and $\beta$ respectively. Let $C$ be the predicate which is true if
\begin{itemize}
    \item $\phi_\beta \models \phi_\alpha$, and
    \item $\alpha_R(w) = \bm{r}_i \Leftrightarrow \beta_R(w) = \bm{r}_i$, for any $w\in\world$ and $i\in\{1,\dots,m\}$.
\end{itemize} Then $C$ is a sufficient condition for actions.
\end{lemma}

This theorem states that if we replace a single action $\alpha$ by another action $\beta$ which has the same return values and whose model $\phi_\beta$ guarantees all the behavior that was guaranteed by $\alpha$, then the modified structure must still be correct with regard to any specification. Intuitively, this works because if this new $\beta$ is selected it still guarantees the behavior of $\alpha$, and as the return values are unchanged all other action in the decision structure are selected under the same conditions. Importantly, this sufficient condition is \emph{local}, in that it can be checked by considering only the actions and their return values. Recalling Theorem \ref{verificationcorrespondence}, this sufficient condition applies equally to modules. 
As a practical example, for BTs (and similarly for $k$-BTs) this means that any subtree can be replaced by a subtree with the same Success and Failure conditions and with stronger guarantees while preserving correctness (this is proved specifically for BTs in~\cite{biggar2020framework}). 
We explore this in the following example.

\section{Illustrative Example} \label{sec:example}

We now demonstrate our results on a concrete robotic example. This example is inspired by the work of~\cite{klockner2016behavior} on high-altitude pseudo-satellite (drone) control with Behavior Trees. We utilise a decision structure to design the high-level autonomous decision-making, composed from a finite number of actions for the solar-powered drone. We seek to verify if our designs are correct with respect to a system specification represented via an LTL formula.

 The vehicle in this example is equipped with a camera, and is capable of sensing its altitude, location, the local weather conditions and the light levels. This vehicle is intended to complete a reconnaissance task, where it must fly safely while awaiting receipt of a goal location, at which point it must fly to this location and take a photograph of the goal. The vehicle should never run out of power while in the air and should never fly at a low altitude during dangerous weather.
\subsection{Modelling the world, actions and specification}
Formally, we shall describe our world model by the variables shown in Table \ref{tab:variables}.
\begin{table}[h]
    \centering
    \begin{tabular}{c|c c c c}
        Variable & Possible Values & & & \\
        \hline
        Battery  & $b0$ & $bLow$ & $bMid$ & $bHigh$\\
        Weather & $calm$ & $windy$ & $storm$ & \\
        Altitude & $landed$ & $low$ & $high$ & \\
        Light level & $dark$ & $dim$ & $bright$ &\\
        Located at goal& $at$ & $\neg at$ & & \\
        Goal known & $goal$ & $\neg goal$ & & \\
        Have photograph & $photo$ & $\neg photo$ & &
    \end{tabular}
    \caption{Variables describing the world}
    \label{tab:variables}
\end{table}
The variables corresponding to the goal, photograph and location are standard Boolean variables, but the variables for battery level, weather, altitude and light level have more possible values. We treat each value as a Boolean variable and enforce that precisely one is true at any time. Of course, this rule must be enforced somewhere within the logic. We do this, as in~\cite{biggar2020framework}, by introducing additional LTL formulas representing the world model. In this case we enforce the following rules.
\begin{itemize}
    \item Disjointedness: Variables take exactly one value. For instance, we express this for Weather as $(calm \land\neg windy\land \neg storm)\lor(\neg calm\land windy \land \neg storm)\lor (\neg calm\land\neg windy\land storm)$ and similarly for the other variables in Table~\ref{tab:variables}.
    \item Progression: Battery and weather conditions must degrade sequentially. That is, $calm \Rightarrow\X\neg storm$ and $(bHigh \Rightarrow\X(bHigh\lor bMid))\land(bMid \Rightarrow\X(bLow\lor bMid\lor bHigh))$. This enforces that storms must be preceded by windy weather, and that the battery being dead must be preceded by it being at mid-level, then low.
    \item Fairness: It is necessary to assert fairness in many LTL-based schemes. In this case, we would like to ensure the agent is not always prevented from completing the mission due to the battery or weather conditions. Additionally, we assume that a goal is eventually known. We therefore assert that $\event\always calm \land \event \always bright \land \event\always goal$. This allows storms to come an arbitrary large number of times, to which the agent must respond appropriately, but ensures the agent can eventually make progress. Here we have asserted that the light levels become bright, but we have yet to enforce the relationship between light levels and battery levels (caused by the solar panels on the drone). We do this with the following formula $ (bright \lor(dim \land high)) \Rightarrow (\event(\neg b0 \land \neg bLow) \land (bMid \Rightarrow \X(bMid \lor bHigh)))$, which requires that if it is bright, or dim but at high altitude, then the battery eventually climbs to $bMid$ and does not drop below $bMid$ during that time.
    \item Initial Conditions: we will also enforce the initial conditions $ \neg storm \land\neg landed \land bHigh\land \neg at$.
\end{itemize}

Equipped with this world model, we can now describe and construct LTL models of each action that is available to the vehicle. We give these descriptions and models in Table \ref{tab:variables}.
\begin{table}[h]
    \centering
    \begin{tabular}{c||p{4cm} | p{7cm} | p{4cm}}
        Name & Description & Model & Return conditions \\
        \hline
        \textbf{Land} & Land the vehicle. & $\X landed \land((at\land \X at)\lor (\neg at \land \X\neg at))$ & \\
        \textbf{Ascend} & Rise to high altitude, taking off if necessary. & $\X high \land((at\land \X at)\lor (\neg at \land \X\neg at))$ & $\bm{s}$: $high$ \\
        \textbf{Descend} & Go to low altitude, taking off if necessary. & $\X low \land((at\land \X at)\lor (\neg at \land \X\neg at))$ & $\bm{s}$: $low$ \\
        \textbf{Avoid} & Fly safely to avoid dangerous weather, by landing or increasing height. & $((windy \lor storm) \Rightarrow \X(landed \lor high)) \land ((at \land \X at) \lor (\neg at \land \X\neg at))$ & \\
        \textbf{Photograph} & Take a photo. & $(bright\land at) \Rightarrow \event photo$& $\bm{s}$: $photo$, $\bm{f}$: $\neg bright \lor\neg at$, $\bm{m}$: $landed \lor high$ \\
        \textbf{GoTo} & Fly to the goal. & $(high \Rightarrow \X high) \land (landed \lor (\X\neg landed \land \event at))$ & $\bm{s}$: $at$, $\bm{f}$: $\neg at \land windy$, $\bm{m}$: $landed\lor low$\\
        \textbf{Circle} & Circle while maintaining altitude and location &$ ((\neg at \land \X\neg at)\lor (at \land \X at)) \land ((high \land \X high) \lor (\neg high \land \X\neg high))$ & \\
    \end{tabular}
    \caption{Actions available to the vehicle and their models.}
    \label{tab:actions}
\end{table}
There are two points to note now from this table. Firstly, not all actions have return conditions, in which case we assume they return any other unspecified return value. This is equivalent to an action in a BT which only returns `Running'. There are three distinct return values present, $\bm{s}$, $\bm{f}$ and $\bm{m}$. We shall interpret $\bm{s}$ and $\bm{f}$ in their usual BT sense as Success and Failure. We shall use $\bm{m}$ as a third intermediate value, when it is helpful to distinguish more possibilities in the return values. For instance, the Photograph action returns $\bm{m}$ if it is not at low altitude. The additional value is useful here, because in this case the action does not Succeed or Fail per se, but the information is potentially relevant for the structure---for instance, the resultant photographs may be of poor quality. When designing the conditions under which these values are returned we attempt to follow reasonable intuition and to ensure that these variables do not depend on memory, as this would violate reactiveness~\cite{kbts}. The second point is the repeated pattern $((at \land \X at) \lor (\neg at \land \X\neg at))$, which states that the variable $at$ is invariant under that action. That is, when the action is selected the value of that variable must be unchanged in the subsequent state.

In addition to these listed conditions, we shall also use any Boolean formula of variables as shorthand for an action, where it is assumed that this action returns $\bm{s}$ if the formula is true and $\bm{f}$ if it is false, thus corresponding precisely to a Condition node in a BT. We shall also use the actions Battery, Light and Weather (Fig.~\ref{fig:h1 and q1} and \ref{fig:d5}) as three-valued conditions. Specifically, the return conditions are the following. Battery: $\bm{s}$: $bHigh$, $\bm{m}$: $bMid$, $\bm{f}$: $b0 \lor bLow$; Light: $\bm{s}$: $bright$, $\bm{m}$: $dim$, $\bm{f}$: $dark$; Weather: $\bm{s}$: $calm$, $\bm{m}$: $windy$, $\bm{f}$: $storm$. All of the above conditions will be modelled by the formula True.

Next we will present the desired specification of the system formulated in LTL. We will verify the design of the overall system against this specification. The specification is given as the following formula $\varphi = \always ((b0 \Rightarrow landed) \land (storm \Rightarrow (landed \lor high))) \land \event photo$. This formula expresses that the agent should always be landed if its battery level becomes 0, and should always be either landed or at high altitude if there is a storm. In addition, it must eventually have successfully obtained a photograph.

\subsection{Verifying a decision structure}
Now, we propose a decision structure on these actions to control the vehicle. This structure, denoted as $Z_1$, is shown in Figure~\ref{fig:d1}.
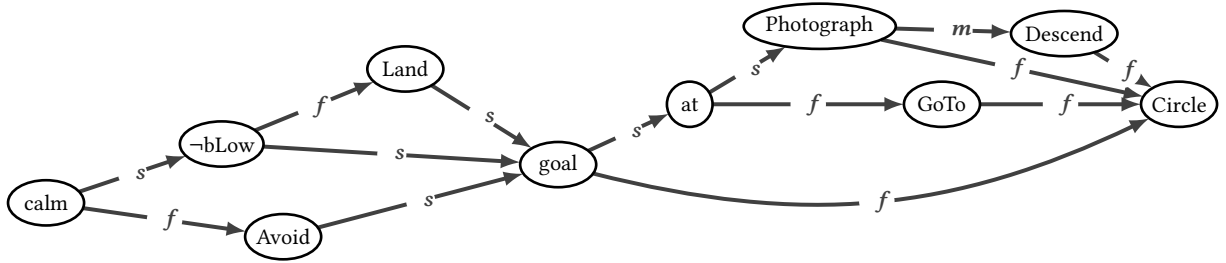
\begin{figure}
    \centering
    \begin{tikzpicture}
\clip (0,0) rectangle (16.0,4.0);
\Vertex[x=2.860,y=1.834,size=1.1,shape=ellipse,style={minimum height=0.6cm,minimum width=1.1cm},opacity=0,fontscale=1.2,label=$\neg$bLow]{!bLow}
\Vertex[x=5.275,y=2.829,size=0.8,shape=ellipse,style={minimum height=0.6cm,minimum width=1cm},opacity=0,fontscale=1.2,IdAsLabel]{Land}
\Vertex[x=0.550,y=1.054,size=0.8,shape=ellipse,style={minimum height=0.6cm,minimum width=1cm},opacity=0,fontscale=1.2,IdAsLabel]{calm}
\Vertex[x=3.668,y=0.605,size=1,shape=ellipse,style={minimum height=0.6cm,minimum width=1cm},opacity=0,fontscale=1.2,label=Avoid]{Safety}
\Vertex[x=7.285,y=1.561,size=1.1,shape=ellipse,style={minimum height=0.6cm,minimum width=1cm},opacity=0,fontscale=1.2,label=goal]{target}
\Vertex[x=9.015,y=2.361,opacity=0,style={minimum height=0.6cm,minimum width=0.6cm},fontscale=1.2,IdAsLabel]{at}
\Vertex[x=10.725,y=3.395,size=1,shape=ellipse,style={minimum height=0.6cm,minimum width=2cm},opacity=0,fontscale=1.2,label=Photograph]{Photo}
\Vertex[x=13.940,y=3.298,size=1,shape=ellipse,style={minimum height=0.6cm,minimum width=1.4cm},opacity=0,fontscale=1.2,label=Descend]{Down}
\Vertex[x=12.332,y=2.361,size=1,shape=ellipse,style={minimum height=0.6cm,minimum width=1cm},opacity=0,fontscale=1.2,IdAsLabel]{GoTo}
\Vertex[x=15.450,y=2.361,size=1,shape=ellipse,style={minimum height=0.6cm,minimum width=1cm},opacity=0,fontscale=1.2,IdAsLabel]{Circle}


\Edge[,label=$\bm{f}$,fontscale=1.2,Direct](!bLow)(Land)
\Edge[,label=$\bm{s}$,fontscale=1.2,Direct](!bLow)(target)
\Edge[,label=$\bm{s}$,fontscale=1.2,Direct](Land)(target)
\Edge[,label=$\bm{s}$,fontscale=1.2,Direct](calm)(!bLow)
\Edge[,label=$\bm{f}$,fontscale=1.2,Direct](calm)(Safety)
\Edge[,label=$\bm{s}$,fontscale=1.2,Direct](Safety)(target)
\Edge[,label=$\bm{s}$,fontscale=1.2,Direct](target)(at)
\Edge[,label=$\bm{f}$,fontscale=1.2,Direct,bend=-20](target)(Circle)
\Edge[,label=$\bm{s}$,fontscale=1.2,Direct](at)(Photo)
\Edge[,label=$\bm{f}$,fontscale=1.2,Direct](at)(GoTo)
\Edge[,label=$\bm{f}$,fontscale=1.2,Direct](Photo)(Circle)
\Edge[,label=$\bm{m}$,fontscale=1.2,Direct](Photo)(Down)
\Edge[,label=$\bm{f}$,fontscale=1.2,Direct](Down)(Circle)
\Edge[,label=$\bm{f}$,fontscale=1.2,Direct](GoTo)(Circle)
\end{tikzpicture}
    \caption{$Z_1$: An initial attempt at a decision structure for high-level control of the drone. There is not necessarily a one-to-one correspondence between return values of the actions and arcs out of the nodes they label, as discussed in Remark \ref{rem:return values and arcs}}
    \label{fig:d1}
\end{figure}
This structure and our knowledge of the action models describe the following decision-making for the drone. Initially, we check whether the weather is calm, selecting the Avoid action if not. Otherwise the battery level is queried, and the vehicle lands if the battery is low. Otherwise, we check for the presence of a goal, and either move to the goal if it exists or take a photograph if the agent is already at the goal's location, descending to low altitude if necessary. While the design intuitively seems correct, as with many cyber-physical systems, this informal analysis is not sufficient for such a safety-critical application. Hence, we will apply the verification scheme of Section \ref{sec:verrification}. We do this by first constructing $\Psi_1$, the LTL model of $Z_1$. We can easily implement a program to construct this from any decision structure using a graph traversal procedure. Next we must incorporate the world model. Let $init$ denote the initial condition formula described above, and let $rules$ be the conjunction of the other rule formulas. The initial conditions hold only in the first state of any execution, but the other rules such as disjointedness must hold in every state. Hence overall we can represent these requirements by the formula $init \land \always rules$. Finally, the actual verification of the structure involves checking the entailment $init \land \always rules \land \always \Psi_1\models \varphi$, which can be performed using off-the-shelf LTL software, for which we use Spot~\cite{spot}.

What we find is that $Z_1$ does \emph{not} satisfy the specification $\varphi$. In fact, Spot returns a counterexample in which a state is reached where $windy \land bLow$ holds. In this case $Z_1$ selects Avoid, but then in the next state $high \land b0$ holds, which violates $\varphi$. Essentially, $Z_1$ did not check the battery level due to $windy$ being true at that time step, and so did not land when required. One potential fix for this problem is shown in the left-hand side of Figure~\ref{fig:h1 and q1}, where the battery levels are checked before the weather condition. However, this modification alone is not sufficient, as verifying the modified structure still produces a counterexample. In this case, low battery levels force the vehicle to land, after which it never takes off again and so never reaches the goal. This can be fixed with the addition of an Ascend-labelled node, as shown in the right-hand side of Figure~\ref{fig:h1 and q1}. The result of both modifications is the structure $Z_2$ in Figure~\ref{fig:motivation}. Performing the verification procedure once more on this structure confirms that it satisfies the specification $\varphi$.

\subsection{Modifying modules}
Let us demonstrate how we can modify $Z_2$ while preserving its correctness. To begin, find all modules and construct its module decomposition using Algorithm~\ref{alg:findmodules}. The decomposition is shown in Figure~\ref{fig:motivation}. We will now demonstrate several results of the paper with this structure. Observe first that the essential complexity is 2---this can be easily calculated from the module decomposition. We can thus immediately conclude that this structure is not structurally equivalent to any $k$-BT. It is likewise easy to see that this is not structurally equivalent to a DT.

Now we make a modification to this decision structure. We know, from the results in this paper, that if the modification is limited to a module, we can check the correctness using the sufficient condition for actions presented in Lemma~\ref{scheme suff condition}.
\begin{remark}[All modification are in modules]
Every possible modification to any decision structure is in fact confined to some module, and there is a unique smallest module containing any such modification. 
Theorem~\ref{verificationcorrespondence} states that checking the correctness of a module modification is on par with modifying actions. However, it should be noted that this verification becomes more complex as the size of the module grows. Of course, if the smallest module containing a modification is the entire structure, then the entire verification is essentially repeated from scratch. However, the modules gives us insight as to which nodes are affected by changes. Suppose changes were made in $Z_2$ to GoTo and Ascend. These are contained in a module by themselves, suggesting such a change has limited impact. By contrast, if GoTo and Photograph are modified, their smallest containing module is \{at, GoTo, Ascend, Photograph, Descend\}, all of which may be impacted by this change.
\end{remark}
Suppose now that we replace the module $H = \{$b0, bLow, calm, Land, bHigh, bright, Avoid\} by the structure $Q$. We shall denote the subgraph $Z_2[H]$ by $K$ and both $K$ and $Q$ are shown in Figure~\ref{fig:h1 and q1}. This constructs a structure $Z_3 = (Z_2/H)\cdot^H Q$. We will show $Z_3$ is correct by considering only $K$ and $Q$, using the equivalence of sufficient conditions for actions and modules. By Lemma~\ref{scheme suff condition} we know it is sufficient to check that $K$ and $Q$ return each possible value under the same conditions and $\Psi_Q\models \Psi_K$. In actuality, as we have that $\always rules$ holds, we only need to check that $\always rules \land \Psi_Q\models \Psi_K$. As with the model, finding the return conditions of a decision structure is straightforward given return conditions of each of its actions. We can therefore easily confirm that both $K$ and $Q$ only return $\bm{s}$, and both do so precisely when $calm \land (bHigh \lor (bMid\land bright))$. Finally we check that $\always rules \land \Psi_Q\models \Psi_K$, and find that this is true. Hence we know that $Z_3$ still satisfies $\varphi$.
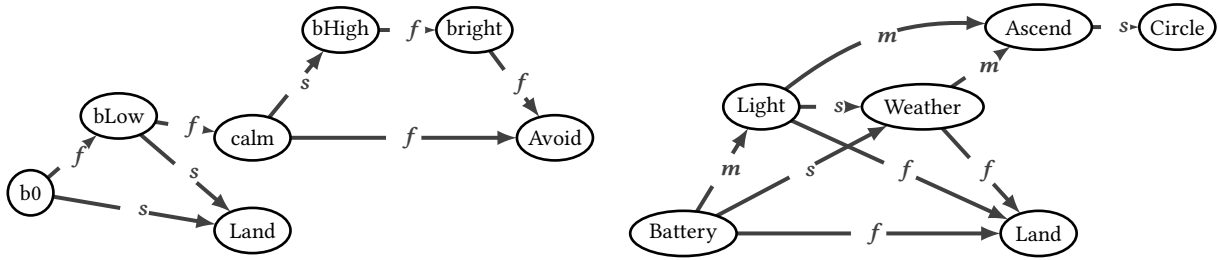
\begin{figure}
    \centering
    \begin{tikzpicture}
\clip (0,0) rectangle (8.0,4.0);
\Vertex[x=1.510,y=2.19,size=1,shape=ellipse,style={minimum height=0.6cm,minimum width=1cm},opacity=0,fontscale=1.2,IdAsLabel]{bLow}
\Vertex[x=0.350,y=1.170,opacity=0,fontscale=1.2,IdAsLabel]{b0}
\Vertex[x=4.430,y=3.330,size=1,shape=ellipse,style={minimum height=0.6cm,minimum width=1cm},opacity=0,fontscale=1.2,IdAsLabel]{bHigh}
\Vertex[x=6.190,y=3.330,size=1,shape=ellipse,style={minimum height=0.6cm,minimum width=1cm},opacity=0,fontscale=1.2,IdAsLabel]{bright}
\Vertex[x=3.270,y=0.670,size=1,shape=ellipse,style={minimum height=0.6cm,minimum width=1cm},opacity=0,fontscale=1.2,IdAsLabel]{Land}
\Vertex[x=3.270,y=1.9,size=0.8,shape=ellipse,style={minimum height=0.6cm,minimum width=1cm},opacity=0,fontscale=1.2,IdAsLabel]{calm}
\Vertex[x=7.250,y=1.9,size=1,shape=ellipse,style={minimum height=0.6cm,minimum width=1cm},opacity=0,fontscale=1.2,label=Avoid]{Safety}
\Edge[,label=$\bm{s}$,fontscale=1.2,Direct](bLow)(Land)
\Edge[,label=$\bm{f}$,fontscale=1.2,Direct](bLow)(calm)
\Edge[,label=$\bm{s}$,fontscale=1.2,Direct](b0)(Land)
\Edge[,label=$\bm{f}$,fontscale=1.2,Direct](b0)(bLow)
\Edge[,label=$\bm{f}$,fontscale=1.2,Direct](bHigh)(bright)
\Edge[,label=$\bm{f}$,fontscale=1.2,Direct](bright)(Safety)
\Edge[,label=$\bm{f}$,fontscale=1.2,Direct](calm)(Safety)
\Edge[,label=$\bm{s}$,fontscale=1.2,Direct](calm)(bHigh)
\end{tikzpicture}
    \begin{tikzpicture}
\clip (0,0) rectangle (8.0,4.0);
\Vertex[x=0.850,y=0.631,size=1.2,shape=ellipse,style={minimum height=0.6cm,minimum width=1.4cm},opacity=0,fontscale=1.2,IdAsLabel]{Battery}
\Vertex[x=4.000,y=2.311,size=1.2,shape=ellipse,style={minimum height=0.6cm,minimum width=1.6cm},opacity=0,fontscale=1.2,IdAsLabel]{Weather}
\Vertex[x=1.875,y=2.311,size=1,shape=ellipse,style={minimum height=0.6cm,minimum width=1cm},opacity=0,fontscale=1.2,IdAsLabel]{Light}
\Vertex[x=5.525,y=0.631,size=0.8,shape=ellipse,style={minimum height=0.6cm,minimum width=1cm},opacity=0,fontscale=1.2,IdAsLabel]{Land}
\Vertex[x=5.525,y=3.369,shape=ellipse,style={minimum height=0.6cm,minimum width=1.4cm},opacity=0,fontscale=1.2,label=Ascend]{Up}
\Vertex[x=7.350,y=3.369,size=1,shape=ellipse,style={minimum height=0.6cm,minimum width=1cm},opacity=0,fontscale=1.2,label=Circle]{Circ1}
\Edge[,label=$\bm{s}$,fontscale=1.2,Direct](Battery)(Weather)
\Edge[,label=$\bm{m}$,fontscale=1.2,Direct](Battery)(Light)
\Edge[,label=$\bm{f}$,fontscale=1.2,Direct](Battery)(Land)
\Edge[,label=$\bm{m}$,fontscale=1.2,Direct](Weather)(Up)
\Edge[,label=$\bm{f}$,fontscale=1.2,Direct](Weather)(Land)
\Edge[,label=$\bm{s}$,fontscale=1.2,Direct](Light)(Weather)
\Edge[,label=$\bm{m}$,fontscale=1.2,Direct,bend=20](Light)(Up)
\Edge[,label=$\bm{f}$,fontscale=1.2,Direct](Light)(Land)
\Edge[,label=$\bm{s}$,fontscale=1.2,Direct](Up)(Circ1)
\end{tikzpicture}
    \caption{The structure $K$ (left) is a subgraph of $Z_2$ (Figure~\ref{fig:motivation}). We will replace it by the structure $Q$ (right).}
    \label{fig:h1 and q1}
\end{figure}
We now make another modification to $Z_3$, as an example of some slightly different analysis. Consider the module $H_2 = $ \{goal, at, GoTo, Photograph, Descend, Ascend, Circle\} in $Z_3$. We shall use $K_2$ as a name for the subgraph $Z_3[H_2]$. We will replace this with a structure $Q_2$, to construct a new decision structure $Z_4 = (Z_3/H_2)\cdot ^{H_2} Q_2$, shown in Figure~\ref{fig:d5}. Once again, we test and find that $\always rules \land \Psi_{Q_2}\models\Psi_{K_2}$. In this case, we do not in fact need to check the return conditions---the above result is a sufficient conditions for actions which are sinks in a decision structure. This can be observed by noting that both modules are selected whenever reached, and their return values do not affect how any subsequent actions are selected. Thus we again conclude that $Z_4$ is still correct with respect to $\varphi$.
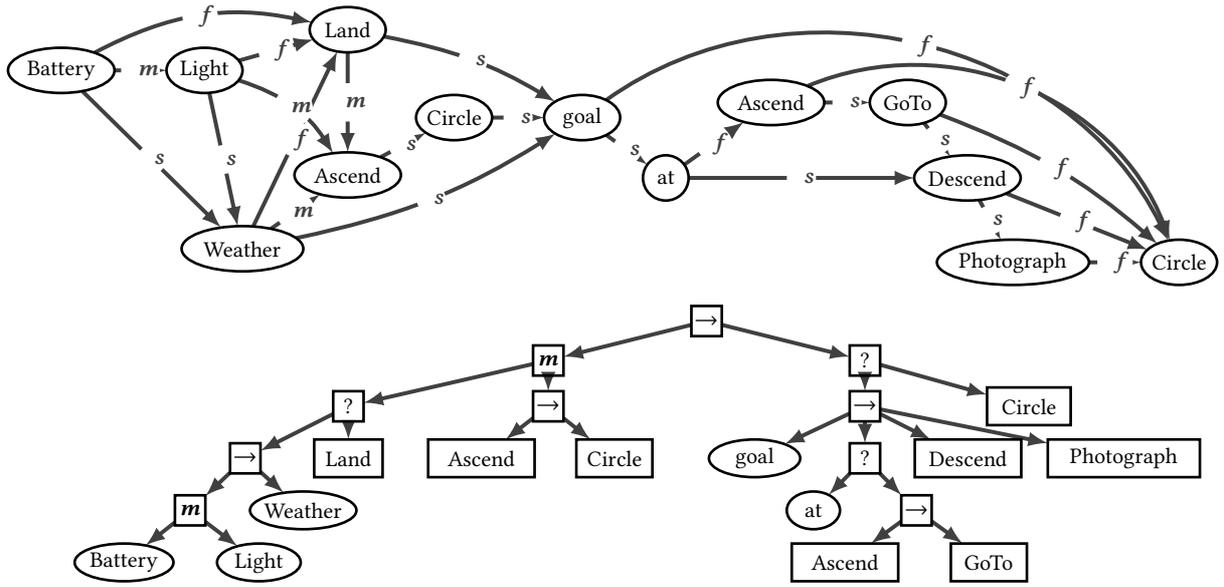
\begin{figure}
    \centering
    \begin{tikzpicture}
\clip (0,0) rectangle (16.0,4.0);
\Vertex[x=0.750,y=2.890,opacity=0,size=1.2,shape=ellipse,style={minimum height=0.6cm,minimum width=1.4cm},fontscale=1.2,IdAsLabel]{Battery}
\Vertex[x=3.133,y=0.525,size=1.2,shape=ellipse,style={minimum height=0.6cm,minimum width=1.6cm},opacity=0,fontscale=1.2,IdAsLabel]{Weather}
\Vertex[x=2.634,y=2.890,size=1,shape=ellipse,style={minimum height=0.6cm,minimum width=1cm},opacity=0,fontscale=1.2,IdAsLabel]{Light}
\Vertex[x=4.517,y=3.429,size=1,shape=ellipse,style={minimum height=0.6cm,minimum width=1cm},opacity=0,fontscale=1.2,IdAsLabel]{Land}
\Vertex[x=4.517,y=1.503,shape=ellipse,style={minimum height=0.6cm,minimum width=1.4cm},opacity=0,fontscale=1.2,label=Ascend]{Up1}
\Vertex[x=5.916,y=2.264,size=1,shape=ellipse,style={minimum height=0.6cm,minimum width=1cm},opacity=0,fontscale=1.2,label=Circle]{Circ1}
\Vertex[x=7.600,y=2.264,size=1,shape=ellipse,style={minimum height=0.6cm,minimum width=1cm},opacity=0,fontscale=1.2,label=goal]{target}
\Vertex[x=8.700,y=1.464,opacity=0,fontscale=1.2,IdAsLabel]{at}
\Vertex[x=13.266,y=0.345,size=1,shape=ellipse,style={minimum height=0.6cm,minimum width=2cm},opacity=0,fontscale=1.2,label=Photograph]{Photo}
\Vertex[x=12.667,y=1.464,size=1,shape=ellipse,style={minimum height=0.6cm,minimum width=1.4cm},opacity=0,fontscale=1.2,label=Descend]{Down}
\Vertex[x=11.883,y=2.464,size=1,shape=ellipse,style={minimum height=0.6cm,minimum width=1cm},opacity=0,fontscale=1.2,IdAsLabel]{GoTo}
\Vertex[x=15.450,y=0.345,size=1,shape=ellipse,style={minimum height=0.6cm,minimum width=1cm},opacity=0,fontscale=1.2,label=Circle]{Circ2}
\Vertex[x=10.084,y=2.464,shape=ellipse,style={minimum height=0.6cm,minimum width=1.4cm},opacity=0,fontscale=1.2,label=Ascend]{Up2}
\Edge[,label=$\bm{s}$,fontscale=1.2,Direct](Battery)(Weather)
\Edge[,label=$\bm{m}$,fontscale=1.2,Direct](Battery)(Light)
\Edge[,label=$\bm{f}$,fontscale=1.2,Direct,bend=20](Battery)(Land)
\Edge[,label=$\bm{s}$,fontscale=1.2,Direct,bend=-10](Weather)(target)
\Edge[,label=$\bm{m}$,fontscale=1.2,Direct](Weather)(Up1)
\Edge[,label=$\bm{f}$,fontscale=1.2,Direct](Weather)(Land)
\Edge[,label=$\bm{s}$,fontscale=1.2,Direct](Light)(Weather)
\Edge[,label=$\bm{m}$,fontscale=1.2,Direct,bend=20](Light)(Up1)
\Edge[,label=$\bm{f}$,fontscale=1.2,Direct](Light)(Land)
\Edge[,label=$\bm{s}$,fontscale=1.2,Direct,bend=10](Land)(target)
\Edge[,label=$\bm{m}$,fontscale=1.2,Direct](Land)(Up1)
\Edge[,label=$\bm{s}$,fontscale=1.2,Direct](Up1)(Circ1)
\Edge[,label=$\bm{s}$,fontscale=1.2,Direct](Circ1)(target)
\Edge[,label=$\bm{s}$,fontscale=1.2,Direct](target)(at)
\Edge[,label=$\bm{f}$,fontscale=1.2,Direct,bend=50](target)(Circ2)
\Edge[,label=$\bm{s}$,fontscale=1.2,Direct](at)(Down)
\Edge[,label=$\bm{f}$,fontscale=1.2,Direct](at)(Up2)
\Edge[,label=$\bm{f}$,fontscale=1.2,Direct](Photo)(Circ2)
\Edge[,label=$\bm{s}$,fontscale=1.2,Direct](Down)(Photo)
\Edge[,label=$\bm{f}$,fontscale=1.2,Direct](Down)(Circ2)
\Edge[,label=$\bm{f}$,fontscale=1.2,Direct,bend=10](GoTo)(Circ2)
\Edge[,label=$\bm{s}$,fontscale=1.2,Direct](GoTo)(Down)
\Edge[,label=$\bm{f}$,fontscale=1.2,Direct,bend=46](Up2)(Circ2)
\Edge[,label=$\bm{s}$,fontscale=1.2,Direct](Up2)(GoTo)
\end{tikzpicture}
    \begin{tikzpicture}
\clip (0,0) rectangle (16.0,4.0);
\Vertex[x=9.237,y=3.600,size=0.4,opacity=0,label=$\rightarrow$,fontscale=1.2,shape=rectangle]{1}
\Vertex[x=7.155,y=3.077,size=0.4,opacity=0,label=$\bm{m}$,fontscale=1.2,shape=rectangle]{2}
\Vertex[x=4.526,y=2.470,size=0.4,opacity=0,label=?,fontscale=1.2,shape=rectangle]{3}
\Vertex[x=3.170,y=1.780,size=0.4,opacity=0,label=$\rightarrow$,fontscale=1.2,shape=rectangle]{4}
\Vertex[x=2.445,y=1.090,size=0.4,opacity=0,label=$\bm{m}$,fontscale=1.2,shape=rectangle]{5}
\Vertex[x=1.576,y=0.400,size=0.7,opacity=0,label=Battery,fontscale=1.2,shape=ellipse,style={minimum width=1.3cm,minimum height=0.5cm}]{6}
\Vertex[x=3.349,y=0.400,size=0.7,opacity=0,label=Light,fontscale=1.2,shape=ellipse,style={minimum width=1.1cm,minimum height=0.5cm}]{7}
\Vertex[x=3.932,y=1.090,size=0.7,opacity=0,label=Weather,fontscale=1.2,shape=ellipse,style={minimum width=1.4cm,minimum height=0.5cm}]{8}
\Vertex[x=4.526,y=1.780,size=0.7,opacity=0,label=Land,fontscale=1.2,shape=rectangle,style={minimum width=0.9cm,minimum height=0.5cm}]{9}
\Vertex[x=7.155,y=2.470,size=0.4,opacity=0,label=$\rightarrow$,fontscale=1.2,shape=rectangle]{10}
\Vertex[x=6.275,y=1.780,size=0.7,opacity=0,label=Ascend,fontscale=1.2,shape=rectangle,style={minimum width=1.4cm,minimum height=0.5cm}]{11}
\Vertex[x=8.024,y=1.780,size=0.7,opacity=0,label=Circle,fontscale=1.2,shape=rectangle,style={minimum width=1cm,minimum height=0.5cm}]{12}
\Vertex[x=11.319,y=3.077,size=0.4,opacity=0,label=?,fontscale=1.2,shape=rectangle]{13}
\Vertex[x=11.319,y=2.470,size=0.4,opacity=0,label=$\rightarrow$,fontscale=1.2,shape=rectangle]{14}
\Vertex[x=9.868,y=1.780,size=0.7,opacity=0,label=goal,fontscale=1.2,shape=ellipse,style={minimum width=1.2cm,minimum height=0.5cm}]{15}
\Vertex[x=11.319,y=1.780,size=0.4,opacity=0,label=?,fontscale=1.2,shape=rectangle]{16}
\Vertex[x=10.641,y=1.090,size=0.7,opacity=0,label=at,fontscale=1.2,shape=ellipse,style={minimum width=0.7cm,minimum height=0.5cm}]{17}
\Vertex[x=11.997,y=1.090,size=0.4,opacity=0,label=$\rightarrow$,fontscale=1.2,shape=rectangle]{18}
\Vertex[x=11.057,y=0.400,size=0.7,opacity=0,label=Ascend,fontscale=1.2,shape=rectangle,style={minimum width=1.4cm,minimum height=0.5cm}]{19}
\Vertex[x=12.949,y=0.400,size=0.7,opacity=0,label=GoTo,fontscale=1.2,shape=rectangle,style={minimum width=1cm,minimum height=0.5cm}]{20}
\Vertex[x=12.675,y=1.780,size=0.7,opacity=0,label=Descend,fontscale=1.2,shape=rectangle,style={minimum width=1.4cm,minimum height=0.5cm}]{21}
\Vertex[x=14.724,y=1.780,opacity=0,label=Photograph,fontscale=1.2,shape=rectangle,style={minimum width=2cm,minimum height=0.5cm}]{22}
\Vertex[x=13.475,y=2.470,size=0.7,opacity=0,label=Circle,fontscale=1.2,shape=rectangle,style={minimum width=1.1cm,minimum height=0.5cm}]{23}
\Edge[,fontscale=1.2,Direct](1)(2)
\Edge[,fontscale=1.2,Direct](1)(13)
\Edge[,fontscale=1.2,Direct](2)(3)
\Edge[,fontscale=1.2,Direct](2)(10)
\Edge[,fontscale=1.2,Direct](3)(4)
\Edge[,fontscale=1.2,Direct](3)(9)
\Edge[,fontscale=1.2,Direct](4)(5)
\Edge[,fontscale=1.2,Direct](4)(8)
\Edge[,fontscale=1.2,Direct](5)(6)
\Edge[,fontscale=1.2,Direct](5)(7)
\Edge[,fontscale=1.2,Direct](10)(11)
\Edge[,fontscale=1.2,Direct](10)(12)
\Edge[,fontscale=1.2,Direct](13)(14)
\Edge[,fontscale=1.2,Direct](13)(23)
\Edge[,fontscale=1.2,Direct](14)(15)
\Edge[,fontscale=1.2,Direct](14)(16)
\Edge[,fontscale=1.2,Direct](14)(21)
\Edge[,fontscale=1.2,Direct](14)(22)
\Edge[,fontscale=1.2,Direct](16)(17)
\Edge[,fontscale=1.2,Direct](16)(18)
\Edge[,fontscale=1.2,Direct](18)(19)
\Edge[,fontscale=1.2,Direct](18)(20)
\end{tikzpicture}
    \caption{(\emph{Above}:) The decision structure $Z_4$. (\emph{Below}:) A structurally equivalent 3-BT. Here we use $\seq$ and $\fb$ for the operators $*_{\bm{s}}$ and $*_{\bm{f}}$ respectively, consistent with the usual notation for BTs. We use $\bm{m}$ for the operator $*_{\bm{m}}$.}
    \label{fig:d5}
\end{figure}
\begin{remark}[Why is this useful?]
It is worth recalling at this point the value of this result. While in this example the modules modified were not significantly smaller than the entire structure, in a larger system this is often the case. In that situation, checking that a single module satisfies a local sufficient condition is a more significant improvement in execution time over repeating the verification overall. Verifying a single module also allows for many optimisations, such as ignoring unused variables, which can greatly improve efficiency. Most importantly, this module replacement was \emph{independent of the specification $\varphi$}. The specification may not be known or set at each stage of the design process, but this replacement preserves correctness for all specifications. In the introduction we flagged the construction of libraries of correct behavior as a use case for such results. This is possible precisely because we can determine which modules can replace other modules without needing to know the specification we will test against. Also, some sufficient conditions may be useful despite not being local, such as the strong form of Theorem 5.7 in~\cite{biggar2020framework}. There a  sufficient condition is used which depends on whether the specification is a \emph{safety} or \emph{liveness} condition, but still provides an improvement for verification. Theorem~\ref{verificationcorrespondence} extends to these cases. However, sufficient conditions are generally not also \emph{necessary} conditions. Thus, if we modify part of a decision structure, every module containing this change may fail to satisfy a local sufficient condition even if the overall modified structure does satisfy the specification.
\end{remark}
\begin{remark}[Ad hoc sufficient conditions]
This second example showed how sufficient conditions for actions are very straightforward to construct with intuition once existing sufficient conditions are known. The informal argument given for ignoring the return conditions for $H_2$ and $Q_2$ is an example. This is truly a sufficient condition for actions, following steps similar to the proof of Lemma~\ref{scheme suff condition}, though we have omitted this proof for the sake of brevity. Our results show that we can apply this intuition equally to modules to obtain powerful analysis tools.
\end{remark}
Figure~\ref{fig:d5} depicts the resultant decision structure $Z_4$. Performing a module decomposition reveals that the essential complexity of $Z_4$ is now 1. In other words, in the process of refining $Z_4$ we have been able to remove subparts which were particularly complicated (and thus prone to conceptual errors) with better-structured replacements, while preserving the correctness. Further, as $Z_4$ has complexity 1, it is now structurally equivalent to a $k$-BT, which can be found from its module decomposition. This $k$-BT (which uses three return values, so is a 3-BT) is also shown in Figure~\ref{fig:d5}.
\begin{remark}[Return values and arcs] \label{rem:return values and arcs}
There need not be a one-to-one correspondence between return values of an action labelling a node and the labels on the arcs out of that node. For instance, consider the action GoTo in $Z_4$. GoTo can return $\bm{m}$ if the vehicle is landed or at low altitude, but there is no $\bm{m}$ arc out of GoTo in $Z_4$ so this value is ignored. Conversely, there are $\bm{m}$ and $\bm{s}$ arcs out of Land in $Z_4$, despite the fact that Land never returns these values. Omitting these arcs would produce an identical ASM, but they are included because without them the structure has greater essential complexity, and thus is no longer a 3-BT. In~\cite{generalise} the authors show that a BT can be constructed which is (structurally) equivalent to any DT. Actions in DTs are assumed to not return values, this translation consists of adding additional arcs to the decision structure which do not correspond to values returned by the actions---in the process reducing the essential complexity from 2 to 1. 
\end{remark}
\section{Conclusions and future work}  \label{sec:conclusions}
In this paper we presented decision structures as reactive control architectures, and showed how modules in decision structures capture modularity in reactive decision-making. We showed that modules are easily computed and can be used to analyse the complexity of decision structures, and to characterise BTs, $k$-BTs, DTs and TRs. Finally we showed how these findings allow verification schemes to be applied in a modular way to decision structures. This paper introduces a number of new concepts, so there are significant avenues for future work. One such direction would be to find other architectures realisable as decision structures. Another, as mentioned in the above Remark, would be to examine how arcs can be added to structures to minimise their essential complexity. A different important goal would be to apply the module concept and complexity measure in an appropriate way to Finite State Machines. While Hierarchical FSMs exist as the `modular' form of FSMs, as all FSMs are HFSMs the additional `structuredness' is not strictly enforced. By contrast, BTs, structured programs, object-oriented programs and control-flow graphs with low essential complexity (in McCabe's sense~\cite{mccabe1976complexity}) all achieve modularity by enforcing a specific structure. As a result some have argued for BTs over HFSMs~\cite{btbook,isla2015handling} as BTs are more `modular'. However, if there existed an appropriate module definition for FSMs we could enforce that FSMs be `modular' by requiring low essential complexity (in the sense of this paper), which would correspond closely with the ideas of modularity in structured programming and in this paper. 
\section*{Acknowledgements}This work is partially supported by Defence Science and Technology Group, through agreement MyIP: ID10266 entitled ``Hierarchical Verification of Autonomy Architectures'' and the Australian Government, via grant AUSMURIB000001 associated with ONR MURI grant N00014-19-1-2571. We would like to thank Petter \"Ogren for his careful reading of the paper and insightful comments regarding cyclomatic complexity.
\appendix
\section{Algorithms}
The following algorithm (Algorithm~\ref{alg:findmodules}) finds the set of all modules in a decision structure. Its correctness and time complexity are given by Theorem~\ref{moduletimecomplexity}.
\begin{algorithm}
\SetAlgoLined
\SetKwInOut{Input}{Input}
\SetKwInOut{Output}{Output}
\SetKwData{Mdl}{m}
\SetKwData{Vrt}{v}
\SetKwData{Dsj}{disjoint}
\SetKwData{Comp}{completed}
\Input{ A decision structure $Z$ with $n$ nodes and $\ell$ arc labels}
\Output{ The set of modules of $Z$ of size at least two (smaller modules are trivial)}
Add $\ell$ additional nodes $R_c$ to $Z$ for each label $c$, and for node in $Z$ lacking a $c$-arc we add one to $R_c$\;
 $\Comp \gets \emptyset$ \tcc*[r]{\Comp will store the set of modules}
 Let $T$ be an array of nodes of $Z$ in topological order\;
 \For{$i\gets 0$ \KwTo $n-1$}{
 $\Vrt \gets T[i]$ \tcc*[r]{On this iteration we construct all modules with source \Vrt}
  $S := \{ w : \text{in-degree}(w)-1$ for each successor $w$ of $\Vrt$\} 
  \tcc*{$S$ stores the successors of all visited nodes on this iteration, along with the number of their predecessors which are yet to be visited.}
  $M \gets \{\Vrt\}$ \tcc*[r]{$M$ stores the elements of the current module}
  \For{$j\gets i+1$ \KwTo $n-1$}{
  $w \gets T[j]$\;
  \If{$w\in S$}{
    \If(\tcc*[f]{All of the predecessors of $w$ have been visited from $v$}){$S[w] = 0$}{
      Remove $w$ from $S$, add $w$ to $M$\;
      \For{each successor $t$ of $w$}{
        \eIf{$t$ is in $S$}{
            Decrement $S[t]$ \tcc*[r]{If $t$ is in $S$, then we have visited another predecessor}
        }{
            Set $S[t] = \text{in-degree}(t)-1$ \tcc*[r]{Otherwise, we add $t$ to $S$.}
        }
     }
     \If{$|S| = \ell$}{
        Add $M$ to \Comp \tcc*[r]{$M$ is a module when it has exactly $\ell$ successors}
     }
    }
  }
  }
 }
 \Return{\Comp}\;
 \caption{FindModules}
 \label{alg:findmodules}
\end{algorithm}
\section{Proofs}
\subsection{Section \ref{sec:decisionstructures} proofs}

\begin{proof}[Proof of Lemma \ref{isomorphism of digraphs}]
Let $G$ and $H$ be the decision structures without actions labelled, and let $\bm{r}\in\returnvals$ be some return value which is not the label of any arc in $G$ or $H$. There is no loss of generality by assuming the existence of such a value, because we have allowed the return value set $\returnvals$ to be of arbitrary (possibly infinite) cardinality. For any return value $\bm{j}\in\returnvals$ we shall write $\alpha_{\bm{j}}$ to represent an action $\alpha$ with constant $\bm{j}$ return value.
($\Rightarrow$) Assume $|G|=|H|=n$. Let $G,H\in \Swt$ be structurally equivalent. Hence there exists ordering of $N(G)$ and $N(H)$ such that applying any actions $\alpha_1,\dots,\alpha_n$ in this order makes $G=H$ as ASMs. Construct a graph isomorphism $\Gamma:N(G)\to N(H)$ as follows. Let $s_G$ and $s_H$ be the sources of $G$ and $H$ respectively. Then applying a sequence of actions $\alpha_{\bm{r}}$ which all have constant $\bm{r}$ return values always selects $s_G$ in $G$ and $s_H$ in $H$, so these nodes must always have the same label. Hence set $\Gamma(s_G)=s_H$. Fix a topological order on $G$. Let $v$ in $G$ receive the $\bm{j}$ arc from $s_G$. Then apply the actions $\alpha_{\bm{j}}$ to $s_G$, $\beta_{\bm{r}}$ to $v$ and $\gamma_{\bm{r}}$ to all other nodes. The resultant ASM always selects $\beta$. Hence applying these same actions in permuted order to $H$ gives an ASM that always selects $\beta$, which labels some node $h$. However, the source of $H$ must be labelled by $\alpha$, so $h$ is not the source. $\alpha$ always returns $\bm{j}$, and all other actions always return $\bm{r}$, so $h$ must be the head of the $\bm{j}$ arc out of $s_H$. Set $\Gamma(v)=h$. Applying this argument inductively in topological order constructs an isomorphism $\Gamma:N(G)\to N(H)$, and the actions are applied to corresponding nodes in the isomorphism.

($\Leftarrow$) Let $G,H\in\Swt$ be isomorphic arc-labelled graphs. Then there exists an ordering of both node sets such that for any $\alpha_1,\dots,\alpha_n\in\Act$ we obtain identical isomorphic graphs with identical node labellings and hence identical ASMs as the ASM is derived from the structure.
\end{proof}
\begin{proof}[Proof of Corollary~\ref{uniqueconstructionmap}]
Suppose that there existed another construction map $\mu$ for $T$. Then for any $x\in T$ and any $\alpha_1,\dots,\alpha_n\in\Act$, $x$ is structurally equivalent to $\kappa_T(x)$ and  $\mu(x)$. Then $\kappa_T(x)$ and $\mu(x)$ are structurally equivalent, and hence are isomorphic as labelled graphs for any $x$, so $\kappa_T$ is unique.
\end{proof}
\begin{proof}[Proof of Lemma \ref{bt_construction_map}]
This gives a decision structure from any BT, because the result is always acyclic (because nodes only pass ticks to nodes to their right in the original BT) and has a single source given by the leftmost action in the BT. Additionally, it depends only on the order nodes are ticked in the tree, so if two structurally equivalent BTs are fed into this map the outputs are clearly isomorphic. We need to show now that the ASM is the same as for the original BT. By construction, we begin by ticking the leftmost child, which is the source, and continue to subsequent nodes when the child returns Success or Failure and new nodes are ticked. Eventually either the tree returns Success or Failure, in which case the last node ticked returned that value, or some leaf node returned Running and so was selected. These nodes are exactly those selected by the decision structure, by construction.
\end{proof}
\subsection{Section \ref{sec:modularity} proofs}
\begin{proof}[Proof of Lemma \ref{trivialmodules}]
Clearly $N(Z)$ is a module, as there are no vertices not in it. Likewise $\emptyset$ is a module as there are no arcs into or out of it. Let $v\in N(Z)$. All arcs into this subgraph go to the source ($v$) and there can be no more than one arc out with any given label, so it is a module.
\end{proof}
\begin{proof}[Proof of Theorem \ref{moduletimecomplexity}]
We prove that Algorithm~\ref{alg:findmodules} is correct and has time complexity $O(n^2k)$.
First, we discuss the time complexity. Adding auxiliary nodes $R_c$ and arcs to them is $O(nk)$. Topological ordering $Z$ takes $O(n+m)$ time, where $m$ is the number of arcs. Constructing $S$ is $O(k)$, and finally iterating through the successors of a node $w$ takes $O(k)$ time. Given we iterate through the nodes in a nested fashion, we obtain overall time complexity of $O(n^2k)$.

Now we show that, fixing a node $v$ for the outer iteration, the sets added to $completed$ are precisely the modules of size at least two with source $v$. Since we iterate over all nodes in the outer iteration, all modules are output. Let $M$ be the set of nodes at some point in the inner iteration and let $w$ be the subsequent node in the topological order. We show first that every node in $M$ is reachable only through $v$. This is true trivially for $v$ itself. Assume true currently for $M$ for induction. The subsequent node $w$ is added to $M$ only if $w\in S$ ($w$ is reachable from $v$) and $S[w] = 0$ (every predecessor of $w$ is in $M$). By the inductive hypothesis, all nodes in $M$ are reachable only through $v$, and thus so is $w$, completing the proof by induction. Since all arcs into $M$ go to $v$, $M$ is a module if and only if it has one successor for each label, which happens exactly when it has $\ell$ successors. Thus every set added to $completed$ is a module.

Now we show all modules are added to $completed$. Let $X$ be a module with source $v$. If $M=X$, then $M$ has exactly $\ell$ successors and so is output. Suppose now for induction that $M\subset X$ and all remaining nodes in $X$ have not yet been visited. Again let $w$ be the subsequent node in the inner loop. If $w\not\in S$ or $w\in S$ but $S[w] > 0$, then $w$ is either not reachable from $v$ or reachable on a path not through $v$, by the previous argument. In these cases, $w\not\in X$ and $w$ is not added to $M$, so the inductive hypothesis still holds. If $w\in S$ and $S[w] = 0$, then all predecessors of $w$ are in $M$. If $w$ were not in $X$ then as it is reachable from some node of $M$ and thus $X$ it is reachable from all nodes in $X$. This is impossible because at least one node of $X$ comes after $w$ in the topological order, and so $w\in X$ and $w$ is added to $M$. By induction, eventually $M=X$ and so all modules are constructed.
\end{proof}
\begin{proof}[Proof of Lemma \ref{prime if maximal}]
(\emph{Claim: Z/P has no pairs of parallel arcs}) Suppose otherwise, with $(u,v)$ a pair of nodes with an $\bm{a}$ arc and $\bm{b}$ arc from $u$ to $v$. Since $Z$ has no such arcs, either $u$ or $v$ must be a non-trivial module in $Z$. There are then definitely a pair of arcs out of $u$ to the source of $v$ in $Z$; we will refer to this node as $\alpha_v$. Now let $q$ be the last node in a topological order of $u$. All nodes in $u$ have $\bm{a}$ and $\bm{b}$ arcs, but neither of those arcs out of $q$ can go to $u$ because this would violate the topological order. Thus there is an $\bm{a}$ arc and $\bm{b}$ arc from $q$ to $\alpha_v$, contradicting the fact that $Z$ is not a multigraph and so has no parallel arcs. (\emph{Claim: Z/P is acyclic}) Let $P=\{M_1,\dots,M_n\}$. Suppose there is a cycle $M_1, M_2, \dots, M_k,M_1$ in $Z/P$. Then for each $M_i$ there is an arc $m_i \to s_{(i+1)\mod k}$ in $Z$, where $s_{(i+1)\mod k}$ is the source of $M_{(i+1)\mod k}$ and $m_i\in M_i$. However $m_i$ is reachable from $s_i$ in $Z$ for all $i$, so there must be a cycle $s_1,\dots m_1, s_2,\dots, m_2,\dots,s_k,\dots, m_k,s_1$ in $Z$, which is a contradiction so $Z/P$ is acyclic. Now note that if $M\in P$ is a source in $Z/P$ then its source $s_M$ is a source in $Z$, as any arcs into $s_M$ in $Z$ induce arcs into $M$ in $Z/P$. But $Z$ has a single source, so there is at most one in $Z/P$, but as $Z/P$ is acyclic there is at least one and hence precisely one source. Now assume $P$ is a maximal partition. Any non-trivial module in the $Z/P$ is a union of modules in $Z$, as each node in the $Z/P$ corresponds to such a module. If $Z/P$ is not prime, such a union exists, and is itself a module in $Z$, which is not $N(Z)$, contradicting the maximality of the modules it contains.
\end{proof}
\begin{proof}[Proof of Lemma \ref{uniquedecomposition}]
We know that a modular partition always exists, as the one-element trivial modules always form a partition. Suppose that all maximal modules are pairwise disjoint. Then each node is contained in precisely one maximal module, and these are a partition, so there is a unique maximal partition, satisfying case (1). Now suppose that there are maximal modules $M_1$ and $M_2$ with $x\in M_1 \cap M_2$. Let $K=M_1\cap M_2$, $Q=M_1\setminus K$, $H=M_2\setminus K$, and denote their respective sources by $s_1$ and $s_2$. \emph{Claim: exactly one of $s_1$ and $s_2$ is in $M_1\cap M_2$}. Suppose $s_1,s_2\in M_1\cap M_2$. Then $s_1=s_2$ as $M_1$ and $M_2$ can only have one source. In this case, suppose there existed nodes $k\in M_1\cap M_2$ and $h\in M_2\setminus M_1$ with \begin{tikzcd} k \ar[r,"\bm{r}"] & h \end{tikzcd}. As $M_1$ is a module, it must have sinks in $Z[M_1]$, and arcs labelled $\bm{r}$ from all sinks to $h$. As $M_2$ is a module, all arcs into it must go to its source, which is not $h$, so these sinks must themselves be in $M_2$. As the sources and sinks of $M_1$ are in $M_2$, we conclude $M_1\subseteq M_2$. This is a contradiction, as $M_1$ is maximal. By the same argument for $M_2$, we conclude that $s_1\neq s_2$. As $x$ is reachable from both sources, and there cannot be arcs into a module which do not go to its source, we conclude one of the sources is reachable from the other (and not vice  versa). We assume without loss of generality this is $s_2$. As $x$ is reachable from $s_2$, $s_2\in M_1\cap M_2$. And so $s_2\in M_1\cap M_2$ and $s_1\not\in M_1\cap M_2$. \emph{Claim: $M_1\cup M_2 = N(Z)$}. Let $z$ be the final node in a topological order of $M_1\cap M_2$, so no arcs from $z$ go to nodes in $M_1$. For contradiction, consider any node $v\not\in M_1\cup M_2$. $Z$ is connected, so either (1) $v$ has an arc to $s_1$, (2) $v$ has an arc to $s_2$, (3) $v$ receives an arc from $h\in M_1$, (4) $v$ receives arc from $k\in M_2$, where in all cases we will call the label on the arc $\bm{r}$. (2) is impossible, as $s_2\in M_1$ but is not the source. If (3), then as $M_1$ is a module there is an $\bm{r}$ arc from $z$ to $v$, but $z\in M_2 \implies$ every node in $M_2$ has an $\bm{r}$ arc either internal or to $v$. If (4), then either $z$ has an $\bm{r}$ arc to $v$ and so all of $M_1$ and $M_2$ have $\bm{r}$ arcs that are internal or go to $v$, or $z$ has an $\bm{r}$ arc to a node $k\in M_2\setminus M_1$. In the second case, all of $M_1$ has $\bm{r}$ arcs internal or to $k$ and $M_2$ has $\bm{r}$ arcs internal or to $v$. These hold for any $v$, so we conclude $M_1\cup M_2$ is itself a module, so there cannot exist $v\not\in M_1\cup M_2$ as then we would contradict the maximality of $M_1$ and $M_2$. \emph{Claim: $M_1\setminus M_2$, $M_1\cap M_2$, $M_2\setminus M_1$ are modules}. Let $q$ be the final node in a topological order of $M_1\setminus M_2$. As $z\in M_1$ there must be an $\bm{r}$ arc from $z$ to a node $h\in M_2\setminus M_1$. As $M_1$ is a module, all $\bm{r}$ arcs out of it go to $h$. Thus every node in $M_1\setminus M_2$ must have an $\bm{r}$ arc within $M_1$, as arcs to $h$ contradict the modularity of $M_2$. Thus, as $M_1\cup M_2$ and $M_2$ is a module, $q$ has \emph{exactly one} out-arc which goes to $s_2$ with label $\bm{r}$. Thus, all arcs out of $M_1$ have label $\bm{r}$ and go to $h$, as arcs out of $M_1$ induce arcs out of $q$. All arcs into $M_2$ come from $M_1$, so all arcs to $M_2\setminus M_1$ go to $h$. There are no arcs out of $M_2\setminus M_1$, so it is a module. Similarly, all arcs out of $M_1\setminus M_2$ are labelled $\bm{r}$ and go to $s_2$, so it is also a module. All arcs into $M_1\cap M_2$ go to $s_2$ and all arcs out are labelled $\bm{r}$ and go to $h$, so $M_1\cap M_2$ is also a module. Hence the modular partition $P=\{M_1\setminus M_2,M_1\cap M_2,M_2\setminus M_1\}$ gives the quotient $Z/P$ the structure of a path of length two labelled $\bm{r}$. Note that in case (2) there must be no maximal partitions. To see this, let $M$ be a maximal module in a partition which contains the source. $M$ cannot properly contain or be contained in $M_1$, so $M = M_1$, but then all other modules in the partition must be contained in $M_2$, so cannot be maximal. \emph{Claim: For any modular partition $P'$, if $Z/P'$ is a path it is also labelled $\bm{r}$}. We showed earlier there is a node $q\in M_1\setminus M_1$ with a single $\bm{r}$ arc to $s_2$. Suppose $q\in M \in P'$. The only arc out of $q$ is labelled $\bm{r}$ so either there is an $\bm{r}$ arc out of $M$ in $Z/P'$ or $M$ is a sink in $Z/P'$. If $M$ is a sink in $Z/P'$, then $M$ must contain $M_2$, contradicting its maximality. Thus $M$ has precisely one arc labelled $\bm{r}$ in $Z/P'$ so if $Z/P'$ is a path it is labelled $\bm{r}$. \emph{Claim: there exists a unique $P$ giving the path $Z/P$ maximum length.} Let $T=\{T_1,\dots,T_n\}$ and $R=\{R_1,\dots,R_n\}$ be partitions of $Z$ which give a maximum length ($n$) $\bm{r}$-labelled path as a quotient. We assume the $T_i$ and $R_i$ are in the order they occur in this path. Let $i$ be the smallest value such that $R_i\neq T_i$. As all previous are the same, $R_i$ and $T_i$ have the same source, and so by a previous result one must contain the other. Suppose $T_i=R_i\cup X$, $X\neq \emptyset$. All arcs out of $R_i$ are labelled $\bm{r}$ and go to one node, and there must be at least one arc to $X$ as $T_i$ is a module, so $X$ has one source and all arcs from $R_i$ go to that source. But then as $T_i$ is a module $X$ is a module, with all arcs out of $X$ labelled $\bm{r}$ and going to the source of $T_{i+1}$. Then the modular partition $T'=\{T_1,\dots,R_i,X,T_{i+1},\dots,T_n\}$ gives a $(n+1)$-length path as the quotient $Z/T'$, which contradicts our assumption that $T$ had maximum length. Hence $T$ is unique.
\end{proof}
\begin{proof}[Proof of Lemma \ref{primequotients}]
This follows immediately from Lemmas \ref{prime if maximal} and \ref{uniquedecomposition}.
\end{proof}
\begin{proof}[Proof of Lemma \ref{moduleinmodule}]
Suppose $X$ is a module in $Z$. Then $X$ has a single source in $Y$. As every node in $Y$ is in $Z$, for every node $y\in Y$ with an arc \begin{tikzcd} x \ar[r,"\bm{r}"] & y \end{tikzcd} where $x\in X$, all nodes in $X$ have a $\bm{r}$ arc which is either internal to $X$ or goes to $y$, so $X$ is a module i $Z[Y]$. Suppose now $X$ is a module in $Z[Y]$. Any node in $X$ which is not the source cannot receive arcs from $N(Z)\setminus Y$, as this contradicts the modularity of $Y$ as every node not the source of $X$ cannot be the source of $Y$. We know that for arcs from $X$ to other parts of $Y$, all arcs with the same label go the same nodes. Now suppose we have an arc \begin{tikzcd} x \ar[r,"\bm{r}"] & z \end{tikzcd}, with $x\in X$ and $z\in N(Z)\setminus Y$. Then, all nodes in $Y$ have $\bm{r}$ arcs either internal to $Y$ or going to $z$. There cannot be an arc from $X$ to $Y\setminus X$ labelled $\bm{r}$, as then $x$ must have an $\bm{r}$ arc in $Z[Y]$, which it does not. Hence all nodes in $X$ have $\bm{r}$ arcs either internal or which go to $z$. Thus $X$ is a module of $Z$.
\end{proof}
\begin{proof}[Proof of Theorem \ref{kbt_characterisation}]
($\Rightarrow$) Let $T$ be a $k$-BT that is structurally equivalent to $Z$. Every $k$-BT is structurally equivalent to one in compressed form, and so using Lemma \ref{isomorphism of digraphs} we can assume without loss of generality that $T$ is in this form. If $|T|=1$, then the module decomposition is trivially a path. Suppose now that for all $k$-BTs $T$ with $|T|\leq n$ every module in the module decomposition is a path, and the label on the uppermost quotient's path is $\bm{j}$ if and only if the root operator of the tree is $*_j$. Let $|T|=n$, and let $T$ have root operator $*_i$, with children $T_1,\dots,T_n$. Let $x\in\{1,\dots,n\}$. Whenever $T_x$ is ticked, the first node ticked is the leftmost, and so all arcs into $T_x$ in $Z$. Also, every node in $T_x$ can be subsequently ticked after the leftmost node, so the subgraph formed by $T_x$ in $Z$ has a single source. Whenever $T_x$ is ticked in $T$ but not selected, it must have returned $\bm{i}$ and the subsequent node ticked is the source of $T_{x+1}$. Hence all arcs out of $T_x$ in $Z$ are labelled $\bm{i}$ and go to a single node, so $T_x$ is a module. Thus $P =\{T_1,\dots,T_n\}$ is a module partition, and $Z/P$ is a path labelled $\bm{i}$ of length $n-1$. We want to show this is maximum length. By the proof of Lemma \ref{uniquedecomposition} we know that there is a longer path if and only if some $T_x$ has a quotient which is a path labelled $\bm{i}$. However as $T$ is compressed form, the root operator of $T_x$ is not $*_i$, and by the inductive hypothesis the uppermost quotient is not labelled $\bm{i}$. The proof of Lemma \ref{uniquedecomposition} showed that if any quotient by modular partition of a structure is a path labelled $\bm{r}$, then all quotients which are paths are labelled $\bm{r}$. Hence $P$ is of maximum length, and so $P$ is the first partition in the module decomposition of $Z$. The factors are $Z[T_1],\dots,Z[T_n]$. By the inductive hypothesis, their quotients are again paths, so we are done.

($\Leftarrow$) Let $Z$ be a switching structure where every quotient is a path. Construct a $k$-BT recursively, where the root operator is given by the label on the uppermost path, and the factors along the path become the children of the root from left to right. Repeating this process constructs a unique $k$-BT $T$. At each layer in the decomposition, the label on the factor path is different than the quotient path it is contained in, as otherwise the quotient path could have had greater length. Thus $T$ is in compressed form. However, the previous argument showed how for $k$-BTs in compressed form their module decomposition corresponds directly to their subtrees, so the module decomposition of $\kappa_{k\text{-BT}}(T)$ has the same module decomposition as $Z$. By repeating performing module expansions we can see that two structures with identical module decompositions must be isomorphic, and hence structurally equivalent by Lemma \ref{isomorphism of digraphs}. Thus $T$ is structurally equivalent to $Z$.
\end{proof}
\begin{proof}[Proof of Theorem \ref{dt_characterisation}]
Suppose $Z$'s module decomposition has this structure and has two distinct labels. When we perform a module expansion, it must be on a sink, which expands into a node with arcs to two new sinks. Repeatedly expanding sinks into nodes with arcs to pairs of sinks leaves a binary tree, with two arc labels, which is therefore a DT. Now suppose $Z$ is a DT. In a tree, the modules are precisely the subtrees, so the maximal partition consists of both children of the root node, which are the maximal subtrees, and the trivial module that is the root. Taking quotient recursively by this partition gives the module decomposition, where every quotient is the graph with one source and two sinks.
\end{proof}
\begin{proof}[Proof of Theorem \ref{testingtimecomplexity}]
We will prove for $k$-BTs first, as BTs and TRs are special cases. Firstly, reject if a structure has more than $k$ distinct labels, which can be done in $O(nk)$ time. Assume now $Z$ has at most $k$ labels. If all quotients in the decomposition are paths, then there is a directed path through all nodes. Hence, all modules must be connected subsets of this path. We can find this path, or prove it doesn't exist, in $O(n+m)$ by topological sorting and counting the length of the longest path in that order. By Theorem \ref{moduletimecomplexity}, we can find all modules in time $O(n^2)$, as number of labels is bounded by a constant $k$. Iterate repeatedly through $N(Z)$ in this order. On each iteration, check that a module has an induced subgraph that is a path, and if so contract the structure by that module. If it is a $k$-BT this  process constructs its module decomposition and all quotients must be paths. If not, then at some point there must be no modules whose induced subgraphs are paths, so we return False. One each iteration we reduce the number of nodes by at least one so we visit at most $O(n^2)$. DTs can be identified by a depth-first search on the structure, checking that each node has exactly one predecessor and exactly two or zero successors, and there are exactly two distinct labels.
\end{proof}
\begin{proof}[Proof of Lemma \ref{calculatincyclcomplexity}]
Construct a new graph $Z'$ by adding a new node $q$ and $s$ arcs from every sink of $Z$ to $q$. Then $Z$ has a single source and sink, and can be considered a program graph as in~\cite{mccabe1976complexity}. It is proved there that the cyclomatic complexity of a program graph with $e$ arcs and $m$ nodes is $e-m+2$. $Z'$ has $n+1$ nodes and $a+s$ arcs, so its cyclomatic complexity is $(a+s) - (n+1) +2 = a+s-n-1+2=a+s-n+1$.
\end{proof}
\begin{proof}[Proof of Theorem \ref{essentialtimecomplexity}]
We know by Theorem \ref{moduletimecomplexity} that all modules can be found by an $O(n^2k)$ time algorithm. If we store the number of arcs, sinks and nodes in each module as it is computed, it only requires $O(n)$ additional time to calculate the cyclomatic complexities of each and to find the maximum, so this is still $O(n^2k)$.
\end{proof}
\begin{proof}[Proof of Theorem \ref{essential_characterisation}]
By Theorem \ref{kbt_characterisation}, a structure with $k$ labels is a $k$-BT iff all quotients in the decomposition are paths. Paths have cyclomatic complexity 1, so $k$-BTs have essential complexity 1. Now suppose a structure $Z$ has essential complexity 1. Then each module has cyclomatic complexity 1, so is a path. Hence $Z$ is a $k$-BT.
\end{proof}
\subsection{Section \ref{sec:verrification} proofs}
\begin{proof}[Proof of Theorem \ref{contraction}]
Let $w\in\world$ be a state. Let $p_w$ and $p_w'$ be the paths in $Z$ and $Z/H$ respectively traversed on input $w$ to $Z(w)$ and $(Z/H)(w)$. There are three cases. In case (1), $p_w$ does not contain any nodes of $H$. As each node and arc on $p_w$ is unmodified in $Z/H$, $p_w=p_w'$ and so $Z(w) = (Z/H)(w)$ and $Z_B(w)=(Z/H)_B(w)$. In case (2) there is at least one node in $H$ on $p_w$ but $Z(w)\not\in H$. Then $p_w = v_1,\dots, v_m, s_H,\dots,h,q,\dots,Z(w)$, where $s_H$ is the source of $H$, $h\in H$ and $v_1,\dots,v_m,q,\dots,Z(w)\not\in H$. Let $\bm{r}$ be the label on the arc $h\to q$. This form is general, because if there are any nodes in $H$ on $p_w$ then they must form a single subpath beginning with $s_H$. As $s_H,\dots,h$ is entirely within $H$, $H(w)=h$. As $v_1,\dots,v_m$ are unchanged, $H$ must be reached in $p_w'$. $H$ therefore returns $\bm{r}$, and must have an $\bm{r}$ arc to $q$. The nodes after $q$ are also unchanged, so $p_w'=v_1,\dots,v_m,H,q,\dots,Z(w)$, and so $Z_B(w) = (Z/H)_B(w)$. In the final case (3), $Z(w)\in H$. Let $\bm{r}_1,\dots,\bm{r}_n$ be the return values on arcs out of $H$. Then for all $i\{1,\dots,n\}$, $Z(w)_R(w)\neq \bm{r}_i$, as there must be an $\bm{r}_i$ arc out of $Z(w)$ in $Z$. Then $p_w$ has form $v_1,\dots, v_m, s_H,\dots,Z(w)$. $p_w'$ must also have prefix $v_1,\dots,v_m$, and the return value of $H$ does not match any out-arc in $Z/H$, so $H$ is selected. However $H(w) = Z(w)$. But then $Z_B(w) = Z(w)_B(w) = H(w)_B(w) = (Z/H)(w)_B(w)=(Z/H)_B(w)$.
\end{proof}
\begin{proof}[Proof of Theorem \ref{verificationcorrespondence}]
Let $C(Z,v,\alpha,\beta)$ be a sufficient condition for actions. Then $C'(Z,H,Q) = C(Z/H,H,(H_B,H_R),(Q_B,Q_R))$ is a sufficient condition for modules. Suppose $C(Z/H,H,(H_B,H_R),(Q_B,Q_R))$ is true. There is a single node $H$ in $Z/H$, labelled by action $(H_B,H_R)$, which is replaced by $(Q_B,Q_R)$. The resultant structure is $((Z/H)\cdot^H Q)/N(Q)$. If $Z/H$ satisfies a specification $\varphi$ according to $V$, then so does $((Z/H)\cdot^H Q)/N(Q)$. However by Theorem \ref{contraction}, $Z_B=(Z/H)_B$ and $((Z/H)\cdot^H Q)_B = (((Z/H)\cdot^H Q)/N(Q))_B$. Thus $Z$ satisfies $\varphi$ whenever $Z/H$ does, and likewise for $(Z/H)\cdot^H Q$ and $((Z/H)\cdot^H Q)/N(Q)$. Thus $C'$ is a sufficient condition for modules. The converse can be shown by reversing the above translation. 
\end{proof}
\begin{proof}[Proof of Lemma \ref{psi canonical}]
Let $w_1,w_2.w_3,\dots\in\world$ be a sequence of states. ($\Rightarrow$) Suppose $w_2,w_3,\dots \in w_1^M$. Then $\exists \alpha_i\in A$ s.t. $w_2,w_3,\dots \in w_1^{\alpha_i}$. By equivalence, $\phi_i$ holds in $w_1,w_2,w_3,\dots$. Also, $M^{-1}(\alpha_i)$ holds in $w$, so $\phi_i\land M^{-1}(\alpha_i)$ holds in this sequence and so $\Psi_M$ holds. ($\Leftarrow$) Suppose $w_1,w_2,w_3,\dots\models\Psi_M$. Then, given the preimage of $M$ is a partition, there exists a unique $j\in\{1,\dots,n\}$ such that $\phi_j \land M^{-1}(\alpha_j)$ holds. Given $\phi_j$ holds, by equivalence we know $w_2,w_3,\dots \in w_1^{\alpha_j}$. However, $M(w_1) = \alpha_j\implies w_1^{\alpha_j} = w_1^M$, and so $\Psi_M$ is equivalent to $(M_B,M_R)$.
\end{proof}
\begin{proof}[Proof of Theorem \ref{verification works}]
Let $w_0,w_1,\dots\in\world$ be a sequence of states generated by $\mathcal{W}$. For each $i$ the signal is given by $\alpha_B(w_i)$, so $w_{i+1},w_{i+2},\dots \in {w_i}^\alpha$. Hence $\phi_\alpha$ holds in $w_i,w_{i+1},\dots$. Thus $\always\phi_\alpha$ holds in $w_0,w_1,\dots$. Then $\varphi$ holds in $w_0,w_1,\dots$.
\end{proof}
\begin{proof}[Proof of Lemma \ref{scheme suff condition}]
Assume $C$ holds. Let $Z'$ be the structure formed by replacing $\alpha$ by $\beta$ on $v$. Let $w\in\world$ be a state. Suppose first that $Z(w)\neq \alpha$. Consider the unique path $p_w$ from the source of $Z$ to $Z(w)$. If $\alpha$ is not on this path, all nodes on this path are unchanged in $Z'$, so $Z'(w)=Z(w)$. If $\alpha$ is on this path, then $\alpha_R(w)=\bm{r}_i$ for some $i$, but also $\beta_R(w)=\bm{r}_i$, and so the same arc is taken at every step on this path in $Z'$, and $Z'(w)=Z(w)$. Now suppose $Z(w)=\alpha$. Then $Z'(w)=\beta$. Thus for all actions $\gamma$, $Z'^{-1}(\gamma)=Z^{-1}(\gamma)$. Assume now we have models $\phi_i$ for all actions $\gamma_i$, and that by this verification scheme $Z$ satisfies some specification $\varphi$ using these models. Then $\always \Psi_Z\models\varphi$. However, $\Psi_{Z'}=(Z^{-1}(\gamma_1)\land\phi_1)\lor\dots\lor (Z^{-1}(\gamma_n)\land\phi_n)\lor (Z^{-1}(\alpha)\land\phi_\beta)$. However, as $\phi_\beta\models \phi_\alpha$, then $\phi_\beta \equiv \phi_\alpha\land\phi_\beta$ and so $\Psi_{Z'}=(Z^{-1}(\gamma_1)\land\phi_1)\lor\dots\lor (Z^{-1}(\gamma_n)\land\phi_n)\lor (Z^{-1}(\alpha)\land\phi_\alpha\land\phi_\beta)\models \Psi_Z$. Hence $\always \Psi_{Z'}\models \always \Psi_Z\models \varphi$, so $Z'$ also satisfies $\varphi$.
\end{proof}

\bibliographystyle{ACM-Reference-Format}
\bibliography{references}
\end{document}